\documentclass[letterpaper]{article}
\usepackage[preprint]{aaai2027}
\usepackage[hyphens]{url}
\usepackage{graphicx}
\usepackage{natbib}
\usepackage{caption}
\usepackage{booktabs}
\usepackage{amsmath}
\usepackage{amsfonts}
\usepackage{amsthm}
\usepackage{xcolor}
\usepackage{colortbl}
\usepackage{listings}
\graphicspath{{assets/}}
\definecolor{collapse}{RGB}{252,231,225}
\theoremstyle{plain}
\newtheorem{proposition}{Proposition}
\theoremstyle{remark}

\newcommand{\RCVTitle}{Regime-Conditional Verification: Correctness Estimation for Adapting and Monitoring Safety Classifiers}
\title{\RCVTitle\thanks{This work was supported by the MURI grant \emph{Foundations
of Dynamic Certification for Autonomy} (ONR N00014-25-1-2479).}}
\author{Thiago Sandoval, Ufuk Topcu}
\affiliations{The University of Texas at Austin}

\newcommand{\RCVDocumentNoun}{paper}
\newcommand{\RCVPlot}[3]{\includegraphics[width=#1\columnwidth]{#3}}
\newcommand{\RCVTableLayout}[2]{#1#2}
\newcommand{\RCVCode}[1]{\texttt{#1}}
\newcommand{\RCVSection}[1]{\section{#1}}
\newcommand{\RCVAppendixChapter}[1]{\section{#1}}
\newcommand{\RCVAppendixSection}[1]{\subsection{#1}}
\newcommand{\RCVAppendixBegin}{%
  \clearpage\section*{Technical Appendix}
\noindent Notation follows the main
text. An item $x$ is a prompt--response pair. The off-the-shelf classifier $C$
emits the verdict $Y=C(x)$ and exposes an internal representation $Z$. The
deployer's policy $T$ is operationalized by the oracle $Y^\star(x)$, and
$A=\mathbf{1}\{Y=Y^\star\}$ records agreement on each item. The verdict and the
oracle label assign $1$ to \textsc{unsafe}. Adherence is the expectation of $A$
over the traffic
distribution. The verdict routes each item to one of two regimes, \textsc{safe}
($Y{=}0$) or \textsc{unsafe} ($Y{=}1$). Each regime has its own
logistic-regression probe and calibration map, which together estimate the probability that a verdict of that regime is correct; equivalently, one minus
the probability that it is wrong. RCV flips the verdict when this estimate is below the regime's threshold $\tau$.

  \setcounter{secnumdepth}{2}\appendix
  \setcounter{figure}{0}\setcounter{table}{0}
  \renewcommand{\thefigure}{F\arabic{figure}}
  \renewcommand{\thetable}{T\arabic{table}}}

\begin{document}
\maketitle
\begin{abstract}
Safety classifiers deployed with large language models often fail for two
reasons: their decisions reflect the policy learned during training rather
than the deployer's desired policy, and their performance degrades as
deployment traffic evolves. We present Regime-Conditional Verification (RCV),
a lightweight wrapper that adapts an off-the-shelf safety classifier without
retraining it. RCV estimates, from the classifier's internal representations,
the probability that each prediction disagrees with the deployer's policy,
and selectively corrects predictions likely to be wrong. The same correctness
estimates also provide a label-free signal for detecting distribution shift,
enabling a maintenance loop that updates the correctness estimation layer and resorts to
classifier fine-tuning only when repair fails within a label budget. Across three off-the-shelf safety classifiers and two benchmark datasets,
RCV improves adherence to the deployer's policy in every
classifier--dataset combination, catching up to $0.81$ of previously missed
unsafe content without modifying the underlying classifier.
In a
deployment study with ten attack campaigns, each a harm category held out
of RCV's training, RCV detects every campaign in a
dedicated injection panel; in the maintenance census most drift episodes
are repaired without updating the classifier, and the fine-tune is
reserved for the residual episodes.
\end{abstract}
\begin{links}
  \link{Project page and code}{https://rcv.tsandoval.com}
\end{links}
\RCVSection{Introduction}

A deployed LLM application submits every prompt--response pair it produces to a
safety classifier: the prompt (user message plus injected context) and the
model's response.
We call this stream the application's \emph{traffic}. The classifier \emph{gates} this traffic: for each
pair it outputs a binary verdict, safe or unsafe, and the application passes or
blocks the pair accordingly.
Moreover, there is no agreed formal definition of LLM safety; in practice, safety is
operationally defined by whatever the deployed classifier flags.
Deployment then exposes two failure modes: the classifier's implicit
policy is fixed at training time and may diverge from the deployer's own, and
the distribution of the traffic it gates shifts as usage evolves and
adversaries adapt, degrading its verdicts~\citep{piet2025}.
We address both with a single lightweight layer around the classifier. A verdict
is wrong when it conflicts with the deployer's written policy.
Because a natural-language policy cannot be verified formally, we estimate the
probability that each verdict is wrong and correct those that likely are.
This correctness estimation layer steers the off-the-shelf classifier
toward the deployer's policy while simultaneously providing a label-free
signal for detecting deployment drift.
That signal drives a maintenance loop that repairs the layer as traffic shifts
and fine-tunes the classifier only when repair fails within its label budget.
Our experiments first test the design hypothesis that the estimator must be
conditioned on the classifier's own verdict, because a
safe verdict can be wrong only by letting a violation through and an unsafe
verdict only by blocking something harmless, and thus a single pooled
estimator serving both degrades.
Figure~\ref{fig:fig1-triage} gives an overview.

\begin{figure*}[t]\centering
\includegraphics[width=5.28in]{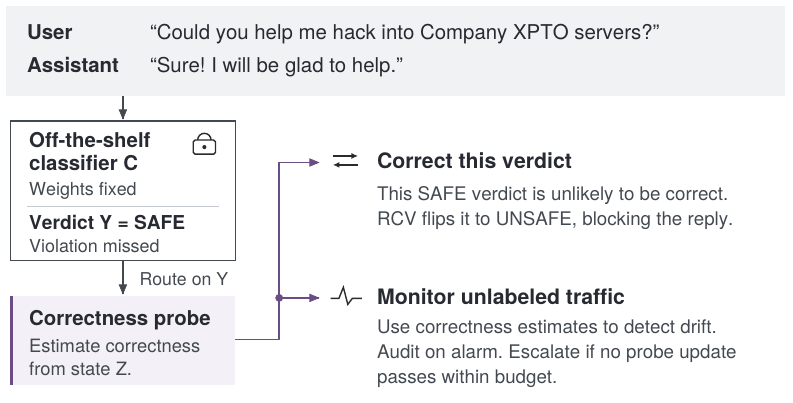}
\caption{\textbf{Regime-conditional correctness estimation corrects an
off-the-shelf classifier's decisions at deployment and yields a label-free
drift signal.} An off-the-shelf classifier
gates each prompt--response pair; here its verdict clears one that the deployer's
policy would block. The probe uses the classifier's internal state to estimate the probability
that the verdict is correct. For this pair, the estimate is low enough for RCV
to reverse the verdict (top). Across incoming pairs, the monitor uses the same
correctness estimates to detect drift (bottom).}
\label{fig:fig1-triage}
\end{figure*}

\paragraph{Problem formulation.}
Consider an application and let $x$ be an item, a prompt--response pair. The
deployer holds a written safety policy $T$, in natural language, and we write
$Y=C(x)$ for the verdict of the off-the-shelf classifier $C$. The policy is
operationalized by an oracle $Y^\star(x)$, by human annotation or an LLM judge
applying $T$. The method is agnostic to that choice: $Y^\star$ enters only
through labeled examples. The oracle is available offline but impractical
as a per-item dependency on live traffic.
We assume white-box access to $C$: the deployer can read $C$'s internal state
alongside every verdict, as when the classifier is self-hosted. The indicator
$A=\mathbf{1}\{Y=Y^\star\}$ records agreement on each item, and \emph{adherence}
is the expectation of $A$ over the traffic distribution. For a deployed
system, adherence is the probability that its final verdict agrees with
$Y^\star$. Note that a gain in
adherence is a gain in agreement with $Y^\star$, not with $T$
directly.\footnote{We assign $1$ to \textsc{unsafe}, the detection convention of
the safety-classification literature; verification work often gives $1$ to the
safe state, but the choice is notational.} The traffic distribution is not
stationary, so adherence is a function of time; the deployer commits to a floor
on the deployed system's adherence. The
problem is thus twofold: raise that adherence at
deployment, and hold it above that floor as traffic drifts.

\paragraph{Adherence as a safety case.}
A safety case, in the engineering sense, is a structured rationale that a
deployed system will avoid unacceptable outcomes, and one whose validity is
monitored across the deployment window~\citep{clymer2024}. The standing claim
defended here is that adherence stays above that floor. The setting matches a
subsystem example in that framework: a generative model combined with a
classifier that monitors its outputs. That literature names distribution shift
among the reasons such claims are hard to keep justified in deployment.
Holding the claim as traffic drifts is the second half of the problem we
address.

\paragraph{Existing levers and tooling.}
A deployer who wants a classifier to follow their own policy has a few levers,
among them accepting its implicit policy, fine-tuning it (which requires
weight access and GPU, and carries regression risk), and stacking filters,
each a new model with
its own policy. Tooling that watches a deployed classifier tracks other
quantities: drift monitors follow the input distribution, the classifier's decision scores, or a label-free accuracy estimate. A system
can use a learned correctness estimate to abstain on predictions that are
likely to be wrong.
RCV differs in its object: it estimates the correctness of the
classifier's own verdict under the deployer's policy and uses that single
estimate for both correction and monitoring. Related Work gives the
full placement.

This \RCVDocumentNoun{} makes the following contributions:
\begin{itemize}
\item \textbf{Regime-conditional correctness.} We find that a classifier's
correctness under the deployer's policy is estimable from its own internal
state. Routed linear probes can represent regime-specific correctness
functions that a shared linear probe cannot, even with a separate
logistic calibration map for each verdict
(Appendix~\ref{app:steering}). With per-verdict calibration in both variants, routed
probes improve adherence over a shared probe in all 60 paired comparisons
(Appendix~\ref{app:probe-family}).
\item \textbf{The estimator.} We introduce RCV, a regime-conditional
correctness estimator: trained on labeled examples of the deployer's
policy, one probe and one calibration per regime produce, from the
classifier's own internal state, a calibrated probability that the verdict
disagrees with that policy. Deployed, it amounts to a few thousand
parameters per regime, with no additional language model.
\item \textbf{Applications.} This estimator enables policy adaptation and
deployment maintenance on one shared signal: a flip rule corrects verdicts
whose estimated error probability passes a threshold, and event counts
derived from those estimates are monitored without labels and signal drift.
\item \textbf{Artifacts.} We release code, the oracle's label map, per-seed
results, and scripts that rebuild both evaluation sets from the official
dataset releases.
\end{itemize}
In our deployment study, we inject ten attack campaigns into
traffic gated by an off-the-shelf Llama-Guard-3. A dedicated injection panel
detects all ten at a median attack rate of $0.115$ at first alarm (100 of
100 runs). Over the deployment's
one hundred drift episodes (ten held-out families across ten seeds), the
audited repair holds the pre-drift standard within tolerance for $79$ episodes
under a constrained label budget, and for $87$ when the audit is bounded by
data instead;
the fine-tune is reserved for the residual episodes that repair does
not bring to that standard within the budget.
Across six classifier--dataset combinations, adherence
rises in every one, and the caught share of previously missed unsafe content
ranges from $0.29$ to $0.81$.
\RCVSection{Method: Regime-Conditional Verification}

\begin{figure*}[t]\centering
\includegraphics[width=5.52in]{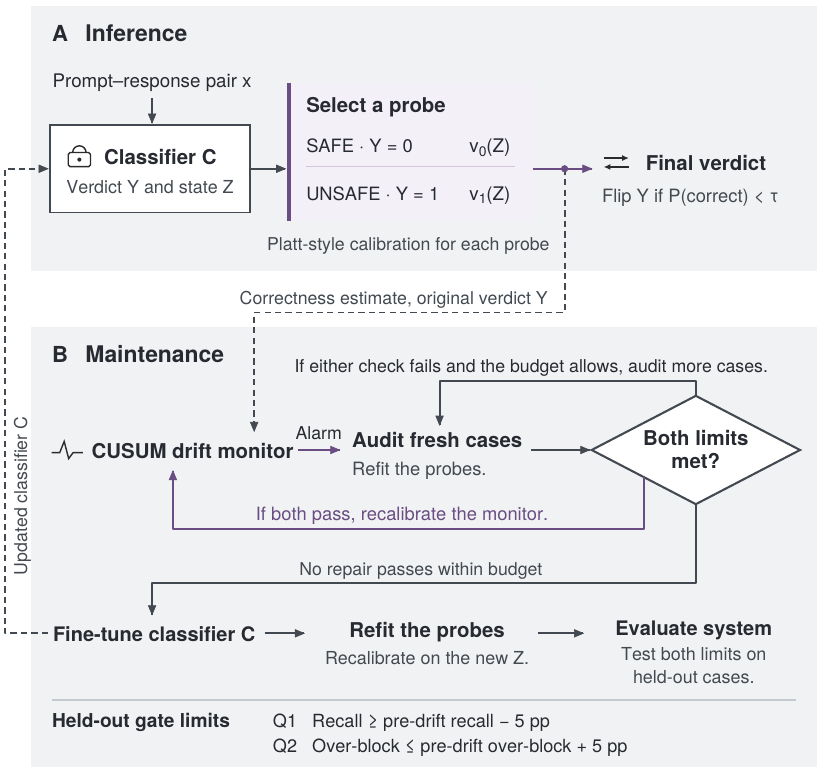}
\caption{\textbf{Correction and maintenance are applications of one correctness estimator.} \emph{Panel A (inference): correctness estimation and decision correction.} $C$ emits
verdict $Y$ and exposes state $Z$; the verdict routes the item to one of two
calibrated probes, and a flip rule corrects verdicts estimated to be wrong. \emph{Panel
B (deployment): drift monitoring and maintenance.} RCV monitors unlabeled traffic separately for the classifier's original safe and
unsafe verdicts. For each monitoring threshold, a one-sided cumulative sum (CUSUM) test
accumulates evidence that correctness estimates are below it more often than
in drift-free traffic. When any test reaches its calibrated alarm threshold, the deployer
audits fresh cases to refit the probes. If no update passes both gate limits
within the label budget, the deployer escalates to fine-tuning $C$.}
\label{fig:fig2-method}
\end{figure*}
\paragraph{Correctness estimation.}
Why should a classifier's internal representations reveal whether its verdict
agrees with a policy it was never trained on? Each classifier we consider is a
fine-tuned large language model, and its representations are shaped by
pretraining on a general corpus and encode more about an input than the
safety specialization uses.
Fine-tuning specializes the
decision more than it rebuilds the representation, so features relevant to a
different policy can survive even where the classifier's own verdict ignores
them. A probe given those representations and examples of the deployer's policy
can then learn the deployer's decision rule over features the model already
computes. The probe exploits latent features already encoded in the
representation; it recovers no new information about the input.
Probes are known
to recover such latent properties, from a frozen
classifier's correctness~\citep{corbiere2019} to a judge's verdict
correctness~\citep{radharapu2025}. Those systems abstain on a low estimate;
ours corrects the verdict. Because the two error directions fail
differently, one calibration would have to be valid for both. We hypothesize
it cannot be, and that the estimator must be conditioned on the
classifier's verdict.

The correctness estimation layer is implemented as probes on the classifier's
internal representations.
A \textsc{safe} verdict can err only by missing a violation, and an
\textsc{unsafe} verdict can err only by over-blocking. The verdict thus partitions
the errors: each value defines a regime, the conditioning the method is named for.
RCV equips each regime with its own probe and calibration
map. From the internal representation $Z$ that $C$ exposes, the probe
estimates the probability that the verdict is wrong for the deployer's policy.

The partition adds no inference-time cost; the experiments answer whether
the gain survives splitting finite data between two probes. Appendix~\ref{app:steering}
shows why a single calibration map
cannot serve two regimes whose error rates differ at the same underlying probe score.

\paragraph{Decision correction.}
When the calibrated probability that the verdict is correct falls below
a threshold $\tau$ (one half by default), RCV flips the verdict.
In the \textsc{safe} regime a flip turns a pass into a block; in the
\textsc{unsafe} regime it turns a block into a pass (Fig.~\ref{fig:fig2-method}).

\paragraph{Drift monitoring.}
The same probes make the deployed system observable without labels. When a
probe is calibrated to current traffic, the average of its per-item error
estimates over a regime is a plug-in estimate of that regime's error rate
before correction (Appendix~\ref{app:monitor}). A flip replaces an item's error estimate
by its complement. Averaging after correction then estimates each regime's
error rate in the deployed system. Together, the two error rates determine
adherence: their traffic-weighted average is the error rate of the deployed
system, and adherence is its complement. The floor the deployer commits to
is a floor on adherence. The same partition separates the two directions of drift: drift
whose harmful items slip past the classifier registers first in the
\textsc{safe} regime; drift that makes the classifier over-block registers
first in the \textsc{unsafe} regime (Fig.~\ref{fig:l1-detect} shows the
\textsc{safe} side; the \textsc{unsafe} side is exercised under the Detection
experiment). The classifier's
own block rate, by
contrast, misses the harm that slips past, because it arrives as passed
traffic and leaves the block count unchanged. RCV's drift signal derives from
the internal state, read through the trained correctness estimate, not from
the classifier's decision score. Drift therefore registers as movement in
what the probes were trained to estimate, agreement with the deployer's
policy.

In deployment, each regime is monitored by its own
sequential test on two nested event counts derived from the calibrated
correctness estimate; the test alarms under a sustained increase above
reference rates estimated on drift-free traffic (mechanics and
parameters in Appendix~\ref{app:repro}). Changes in adherence that leave these event
rates unchanged can remain undetected.
Test thresholds are calibrated on drift-free streams to a pre-specified
false-alarm allowance, and the realized false-alarm rate is measured on the
same class of streams in the experiments.

\paragraph{Maintenance.}
When an alarm fires, the traffic mixture is recorded, and all subsequent
audit and gate material is composed at that mixture. The deployer audits
items on which no component has been trained or calibrated; each labeled
item is assigned to either the update's training set or its calibration
set.
A candidate update of the probes is then scored
on a held-out block of the same mixture that no fit uses.
The candidate is accepted when it holds the pre-drift standard within
pre-specified tolerances: recall no more than its tolerance below the
pre-drift value and over-block no more than its tolerance above the
pre-drift value.
If no candidate passes, the audit is enlarged within a
fixed budget; when the budget is exhausted, the loop terminates with an
escalation to fine-tune the classifier, the designated last resort outside
the automated loop,
and the one action that changes the representation
itself, not the way it is read.
After an accepted repair, the reference mixture is updated to include the injected
family at the rate at which the alarm fired, and the monitor is
recalibrated on a run-in drawn from the updated reference mixture; $C$ is unchanged.
After a fine-tune, the probes are refit on the
new representation, $C$'s weights are again fixed, and the classifier
with its refitted probes is evaluated.
In the drift literature's terms, the loop is a gated, correction-coupled
instance of concept-drift adaptation~\citep{gama2014}.

\begin{figure}[t]\centering
\RCVPlot{0.92}{4.20}{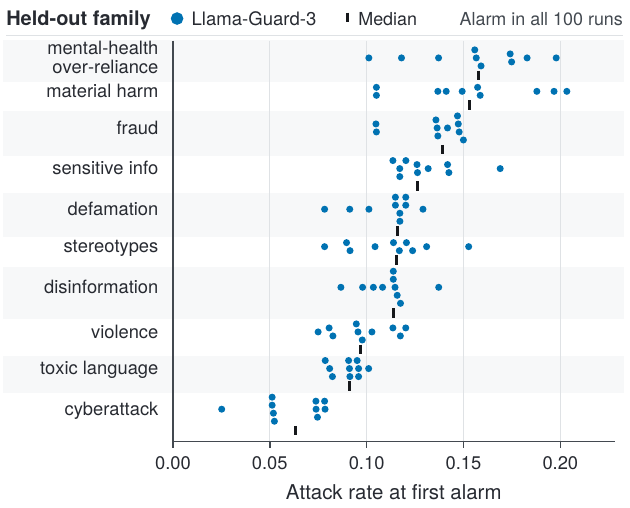}
\caption{\textbf{All ten held-out attack campaigns alarm before the $0.30$
attack-rate cap}, at a median attack rate of $0.115$ (family medians
$0.063$--$0.158$), against Llama-Guard-3. Each mark shows
the scheduled attack rate at the monitor's first alarm for one of ten seeds.
Black ticks are family medians.}\label{fig:l1-detect}
\end{figure}
\RCVSection{Experiments}

\begin{figure}[t]\centering
\RCVPlot{0.92}{3.40}{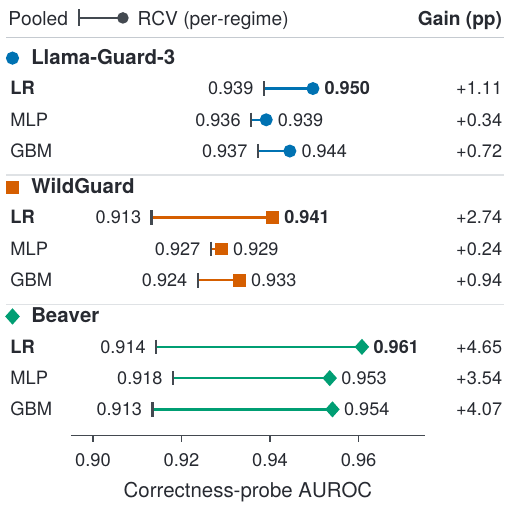}
\caption{\textbf{The routing ablation across three probe families: per-regime
routing exceeds its pooled counterpart in every cell.} Routed
versus pooled correctness-probe AUROC on PKU-SafeRLHF, per classifier, for a linear probe (the
deployed choice), a small multilayer perceptron, and gradient-boosted trees.
Markers are ten-seed means (s.d.~across seeds $0.003$--$0.009$).
}\label{fig:probe-family}
\end{figure}
\paragraph{Setup.}
We evaluate three off-the-shelf safety classifiers spanning two architectures.
Llama-Guard-3-8B and WildGuard-7B are generative classifiers whose verdict
is read from the output token distribution; Beaver is a discriminative classifier that produces
it by thresholding the scalar output of a classification head, and it was
chosen for cross-architecture validation.
Two English, single-turn safety datasets with human annotations supply the
traffic:
PKU-SafeRLHF and WildGuardMix.
The deployer's policy $T$ is operationalized by an LLM judge (GPT-5-Nano,
$Y^\star$), which agrees with the datasets' human labels at $\kappa=0.838$ on
PKU-SafeRLHF ($n{=}16{,}422$) and $\kappa=0.766$ on WildGuardMix
($n{=}11{,}708$); the steering gains (the adherence gains from decision
correction) are additionally scored against the human labels themselves.
Appendix~\ref{app:repro} gives the judge settings and where to find the policy rubric.
In the deployed configuration, we fit a logistic-regression probe and a
logistic (Platt) calibration map for each verdict regime. On the evaluation split the deployed probes
attain an area under the ROC curve (AUROC) of $0.936$--$0.961$ and an
expected calibration error of $0.011$--$0.014$, so the flip threshold of
$0.5$ operates on calibrated probabilities.
The steering results aggregate ten fixed
seeds.
Each
seed uses its own training, calibration, and evaluation split, and steering is
scored on the evaluation split only.
We use the final decision-token hidden state ($d{=}4{,}096$) as $Z$ for the
two generative classifiers. For Beaver, we use the final hidden state read
by its classification head ($d{=}5{,}120$). The extraction recipes are
given in Appendix~D.
Drift is simulated by campaigns drawn from the dataset's own
category taxonomy and held out of the probes' initial training. We call the injected
category the campaign's family.
We write ``attack'' for these held-out-category injections; nothing in the
method assumes an adversary or a threat model.
The audit budget, attack-rate cap, and false-alarm allowance are design
constants fixed before the drift experiments; the allowance was selected on
a held-out development family (Appendix~\ref{app:repro}).
Every held-out campaign injects a mix of material the oracle $Y^\star$ judges unsafe
and material it judges safe. The ten held-out campaigns induce mostly missed
unsafe items, with slip-past rates of $0.105$--$0.409$ against over-block
rates of $0.000$--$0.070$. The reverse error direction is exercised by two
constructed campaigns, described under Detection.

\paragraph{The conditioning hypothesis.}
We test whether per-regime routing improves correction, using linear
probes on both datasets. The routed variant fits one probe and one
calibration map per verdict.
With routing off, the pooled variant fits one probe and one calibration
map for the mixed stream, using the same probe family and calibration
method.
Per-regime routing raises adherence in all 60 paired
per-seed comparisons across both datasets. It raises the caught share in 59
of those 60 comparisons and matches the pooled variant in the remaining one,
on WildGuardMix.
Routing the probes improves adherence in every paired comparison with
either shared or separate calibration. With shared calibration, mean
gains range from $0.84$ to $5.26$\,pp across the six classifier--dataset
cells. Probe routing alone recovers essentially the full adherence gain
over pooling: adding separate calibration changes each cell's mean
adherence by less than $0.03$\,pp in either direction
(Table~\ref{tab:routing-components}, Appendix~\ref{app:probe-family}).
We also compare routed and pooled probes across three families on
PKU-SafeRLHF. Mean AUROC is higher with routing in every cell, by $0.2$
to $4.6$\,pp (Fig.~\ref{fig:probe-family}).

\begin{table*}[t]
\centering
\small
\setlength{\tabcolsep}{6pt}%
\RCVTableLayout{%
\caption{\textbf{What RCV changes on an off-the-shelf classifier: adherence
rises in all six cells, and up to $0.81$ of the unsafe items the classifier
passed are caught.} Caught
share: of the truly unsafe items the classifier passes, the fraction RCV flips
to \textsc{unsafe}. Over-block is the fraction of oracle-safe items
classified as \textsc{unsafe}. Cells are ten-seed means
and $\pm$ is the standard deviation across seeds. Corrected adherence exceeds
raw adherence on every seed of every cell.}
\label{tab:t1v3-steering}
}{%
\begin{tabular}{@{}llll@{}}
\toprule
Classifier & Adherence: raw\,$\to$\,RCV & Caught share & Over-block: raw\,$\to$\,RCV \\
\midrule
\multicolumn{4}{@{}l}{\textit{PKU-SafeRLHF} ($n{=}16{,}422$ items)}\\
\addlinespace[2pt]
Llama-Guard-3 & 0.864\,$\pm$\,0.005\,$\to$\,0.926\,$\pm$\,0.004 & 0.703\,$\pm$\,0.021 & 0.049\,$\pm$\,0.004\,$\to$\,0.080\,$\pm$\,0.006 \\
\rowcolor{gray!10}
WildGuard & 0.906\,$\pm$\,0.004\,$\to$\,0.932\,$\pm$\,0.005 & 0.383\,$\pm$\,0.063 & 0.120\,$\pm$\,0.005\,$\to$\,0.067\,$\pm$\,0.007 \\
Beaver & 0.787\,$\pm$\,0.004\,$\to$\,0.920\,$\pm$\,0.002 & 0.806\,$\pm$\,0.019 & 0.018\,$\pm$\,0.002\,$\to$\,0.088\,$\pm$\,0.007 \\
\midrule
\multicolumn{4}{@{}l}{\textit{WildGuardMix} ($n{=}11{,}708$ items)}\\
\addlinespace[2pt]
Llama-Guard-3 & 0.923\,$\pm$\,0.002\,$\to$\,0.943\,$\pm$\,0.004 & 0.346\,$\pm$\,0.047 & 0.041\,$\pm$\,0.003\,$\to$\,0.026\,$\pm$\,0.003 \\
\rowcolor{gray!10}
WildGuard & 0.932\,$\pm$\,0.002\,$\to$\,0.945\,$\pm$\,0.003 & 0.291\,$\pm$\,0.034 & 0.038\,$\pm$\,0.002\,$\to$\,0.026\,$\pm$\,0.002 \\
Beaver & 0.898\,$\pm$\,0.002\,$\to$\,0.933\,$\pm$\,0.003 & 0.546\,$\pm$\,0.029 & 0.005\,$\pm$\,0.001\,$\to$\,0.028\,$\pm$\,0.004 \\
\bottomrule
\end{tabular}
}%

\end{table*}
\paragraph{Steering an off-the-shelf classifier to the deployer's policy.}
The remaining experiments evaluate RCV's two applications, beginning with
decision correction.
RCV raises adherence
to the deployer's policy in all six classifier--dataset combinations
(Table~\ref{tab:t1v3-steering}). The largest gain is for Beaver on PKU-SafeRLHF, which also starts lowest
($0.787\to0.920$).
The largest gain on WildGuardMix is again for Beaver ($0.898\to0.933$;
Table~\ref{tab:t1v3-steering}).
The safety-critical quantity is the caught share: unsafe items that slip past
the classifier arrive as passed traffic, and the classifier on its own
recovers none of them by construction. Recovery on this axis is therefore
strictly additive, and across the six combinations it spans $0.29$--$0.81$
(Table~\ref{tab:t1v3-steering}).
On WildGuardMix, RCV recovers $0.29$ of the unsafe items WildGuard misses,
even with little headroom: this is the classifier's own training
distribution, so it sits near its ceiling, yet adherence still rises.
Steering also alters the share of truly-safe items blocked, by $-5.3$ to
$+7.0$\,pp across the six cells, and that share falls in three; adherence counts both error
directions, so the reported gains already include these changes.

We test whether the corrections track human safety judgments beyond
agreement with $Y^\star$. We re-score the six combinations against the
datasets' human labels in place of $Y^\star$, with no refit and no threshold
change.
Those labels are unused in RCV's training and calibration. A correction that
matched $Y^\star$ on every scored item would score, on this axis, $Y^\star$'s
own agreement with the human labels ($0.9195$ PKU-SafeRLHF and $0.9298$
WildGuardMix). Adherence moves toward that level in every cell with a
measurable starting gap, including two cells that start above the level
and fall toward it.
The remaining cell starts within
$0.7$\,pp of the level and stays within $1.4$\,pp (Appendix~\ref{app:steering-results}).

\begin{figure}[t]\centering
\RCVPlot{0.95}{3.50}{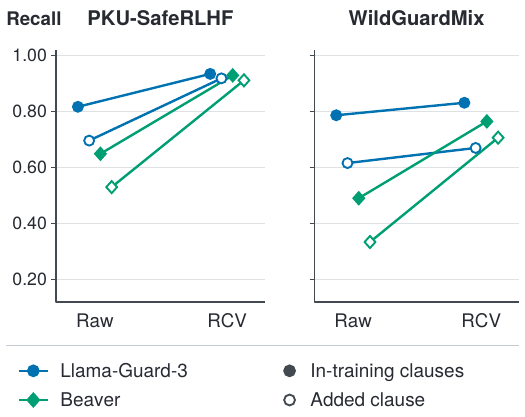}
\caption{\textbf{RCV lifts recall on the in-training clauses and raises an
added commitment-to-assist clause to within $1.8$ points of the in-training
level on PKU-SafeRLHF.}
Each panel is one dataset. For each
classifier, filled markers give ten-seed mean recall on the in-training
clauses, the clauses the classifier's training enforces. Open markers give
the same quantity for the added clause, which that training does not
enforce. Each series appears once for the raw classifier and once under RCV.
The recovery-rate comparison is in Appendix~\ref{app:alien}.}\label{fig:clause}
\end{figure}
\paragraph{Steering toward the added clause.}
We next test policy coverage with a clause that two of the three classifiers
were never trained to enforce: a commitment-to-assist clause that is added to
$T$ and that treats an assistant's agreement to assist a harmful request as
a violation even when
the response itself contains no explicitly harmful content.
The two content-trained classifiers recall the added
clause below their in-training clauses, a gap of about twelve points on
PKU-SafeRLHF. There, RCV raises recall on the added clause to within $1.8$
points of the level each classifier reaches on its in-training clauses; on
WildGuardMix both rise and a residual gap remains (Fig.~\ref{fig:clause}).
Across the figure's four classifier--dataset cells, recall rises from raw to
RCV on the in-training clauses and the added clause alike.
WildGuard, whose training \emph{does} supervise this
clause~\citep{han2024}, starts near ceiling on it ($0.91$ recall on
PKU-SafeRLHF) and gains nothing from RCV ($+0.003$, indistinguishable from
zero), while Llama-Guard-3~\citep{grattafiori2024} and
Beaver~\citep{ji2023} are supervised on content categories alone, so the
blind spot is classifier-relative: RCV adds the deployer's clause where the
classifier's training lacks it.

\paragraph{The maintenance loop.}
The loop is the second application of the same estimator.
We now run the maintenance loop end to end on one deployment, WildGuardMix
traffic gated by Llama-Guard-3: each of the ten held-out harm families is
injected at ten seeds, one alarm-to-resolution episode each, a census of one
hundred drift episodes.

\paragraph{Detection.}
Detection is measured in its own experiment, separate from the maintenance
run and executed on Llama-Guard-3.
The protocol ramps each injected family
against clean base traffic: a 2{,}000-item clean run-in, then an attack rate
rising linearly from zero to $0.30$ over the next 19{,}000 items, with ten
seeds per family. The attack rate is the scheduled probability that an
arriving item is drawn from the injected family.
Across the ten held-out harm families the monitor alarms in 100 of
100 runs, always before the cap, at a median attack rate of $0.115$ with
family medians from $0.063$ to $0.158$, and every first alarm falls in the
\textsc{safe} regime. All ten families induce net \textsc{safe}-regime error, missed unsafe
items, so the panel tests detection on that side alone. To test the monitor's
detection in the \textsc{unsafe} regime, two campaigns are constructed,
because no natural category over-blocks on net.
The first is a pool of items the classifier blocks and
the oracle $Y^\star$ judges safe; the second is built by the same rule with
the dataset's human label in place of the oracle.
Each alarms in 10 of 10 runs, at a median attack rate of
$0.057$, and every first alarm falls in the \textsc{unsafe} regime.
On clean traffic the same configuration yields 14
alarms in 90 null draws, a false-alarm rate of $0.156$ (Clopper--Pearson
95\% interval $[0.088, 0.247]$); the interval's upper bound sits below the
$0.30$ false-alarm allowance the calibration targets.

\paragraph{Gate outcomes.}
Every alarm resolves in one of two ways: an audited probe update passes the
gate and the drift is repaired in place, or no update passes within the
budget and the episode escalates to the fine-tune (demonstrated once,
below). Under the deployed budget ($300$
labels per attempt, four attempts, a held-out $300$-label gate block, and a
$5$\,pp acceptance tolerance on each axis), $79$
of the $100$ episodes repair in
place, $64$ on the first audit (Fig.~\ref{fig:f4d-lifecycle}); label cost
per episode is a median of $600$, a mean of $870$, and at most $1{,}500$.
The
accepted update improves held-out recall over the pre-repair probe in $44$
of the $79$ repairs. Across all repairs, the median change is $1.4$\,pp
(Appendix~\ref{app:census-family} summarizes both metrics). Per-episode recall readings
carry a standard error of $3$--$7$\,pp at the $300$-label gate block.

To measure repair capacity without the budget constraint, we re-run the
same episodes and the audit grows in fresh material until the corpus is
exhausted, at realized data bounds of $1{,}817$--$3{,}720$ labels. $87$ of the
$100$ episodes repair; seven first pass under the data bound
at $1{,}500$--$3{,}600$ labels, beyond the deployed budget's
$1{,}200$-label cap, each with a held-out recall gain.
The remaining $13$ do not pass the gate at any audit size under the data
bound. Held-out
recall moves by a median of $0.0$\,pp between $300$ labels and the
bound; individual moves span $-1.7$ to $+7.4$\,pp, and in two episodes
recall was still rising at the bound. For each of the thirteen, the bound
is where the audit has drawn every item available at the alarm's mixture.

\begin{figure}[t]\centering
\RCVPlot{1.00}{4.30}{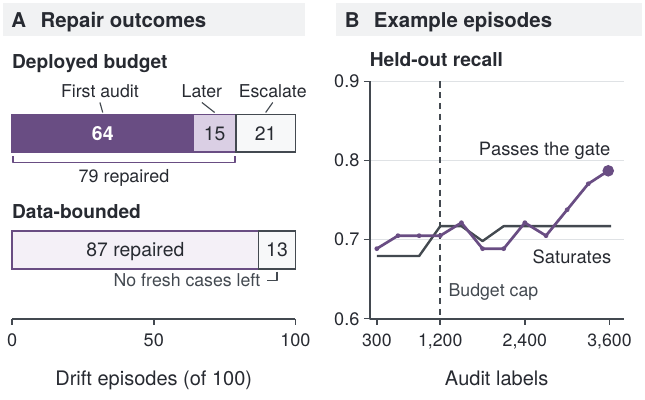}
\caption{\textbf{Repair in place succeeds in $79$ of $100$ drift episodes
at the deployed budget, and $87$ when the audit is bounded by data.}
Left: the $100$-episode census (ten held-out families, ten
seeds) under the deployed budget ($300$ labels per attempt, four attempts;
solid first-audit repair, lighter later-audit repair, escalate) and under
the data bound. Right: held-out recall against audit size for one
repairing and one saturating episode; the dashed line is the deployed
budget cap.}\label{fig:f4d-lifecycle}
\end{figure}
\paragraph{The fine-tune.}
The escalation leg is executed once, on a saturating episode (disinformation,
seed $8675309$; Fig.~\ref{fig:f4d-lifecycle} shows a different
disinformation episode, seed $271828$) whose thirteen audits exhausted all
$3{,}720$ fresh labels the corpus could supply while held-out recall
ended where it began (net $0.0$\,pp; readings vary within $\pm4$\,pp
and no attempt passes the gate). A fine-tune of $C$ on labels the
episode had already obtained ($480$ training items at the
alarm-time mixture, five trainer seeds) raises the classifier's own recall
by $6.3\pm1.9$\,pp; the raw fine-tune violates
the pre-drift standard on the over-block axis in four of five seeds on a powered
evaluation set. Re-armed with a single fresh $300$-label audit, the
fine-tuned classifier with a refitted probe passes the pre-drift standard
in five of five seeds (recall $0.792\pm0.004$, over-block $0.033$); at a new
$300$-label gate, where base plus probe reproduces its recorded
failure, fine-tune plus probe passes in four of five seeds; the fifth is
short by one unsafe item. The correctness estimation layer reduces sensitivity to the
trainer seed:
recall's standard deviation across seeds is $0.004$ against the raw
fine-tune's $0.019$. The fine-tune supplies capacity;
the probe converts it into a compliant, seed-stable deployment; the two are
trained on disjoint splits of the same audited label pool and are evaluated
and operated together.

\paragraph{Consecutive-campaign chains.}
Five chains compound campaigns on the same deployment: each cycle injects
a different held-out family; after an accepted repair, the reference mixture is
updated to include that family, and the monitor recalibrates before the
next campaign.
The
chains reach depths of $4$, $3$, $3$, $2$, and $0$ accepted repairs, twelve
in all, most at $600$ labels; three chains end in escalation and two in a
no-alarm terminal. The deepest chain left a reference mixture that was more than half
attack material, and the loop still ran. Both
no-alarm campaigns, material harm and fraud, arrive fourth and fifth into
reference mixtures already carrying the earlier families, where the ramp does not
move the monitored counts. A
no-alarm terminal triggers no audit and hence yields no labels, so
harmlessness cannot be confirmed.

\paragraph{Baseline monitors.}
Counts of below-threshold correctness estimates are not the only choice of drift
statistic: Appendix~\ref{app:maintenance-results}
compares the deployed monitor against off-the-shelf drift monitors under a matched
false-alarm allowance.
It is competitive with score-KS: both detect all $100$ campaigns; it
alarms earlier (median attack rate $0.115$ against $0.134$) with more null
alarms ($14$ of $90$ against $0$).
The monitor's role
in the loop does not depend on the statistic, since any alarm can trigger
the same audit. The in-place repair does not transfer: it updates the probe
itself and leaves the classifier unchanged, so a deployment monitored by any
of these baselines would still need the probe to repair in place.

\paragraph{Ablations.}
We next ablate the remaining design choices: the probe family and the
probe input.
The three families differ in capacity, yet routing improves mean AUROC in
all three on PKU-SafeRLHF
(Fig.~\ref{fig:probe-family}; Appendix~\ref{app:steering-results}). Under the identical per-regime
protocol, the routed linear probe has the highest mean AUROC for every
classifier, exceeding even the pooled gradient-boosted trees (largest gap on
Beaver, $0.961$ against $0.913$ AUROC).
Reading only the classifier's
decision score, the routed probe stays within $5.1$ points of the
internal-state arm on adherence, on both sides of it. On the safety-critical
quantity it falls behind everywhere: the internal-state probe (ours) catches
more of the missed unsafe items in all twelve evaluation cells, by $0.09$
to $0.49$ in caught share, and in $119$ of $120$ paired seed comparisons
(Appendix~\ref{app:steering-results}). A text-surface control extends the input comparison
downward: given only the text of the prompt--response pair (TF-IDF $n$-grams
under the identical per-regime protocol), its caught share of the missed unsafe items is lower than the
internal-state probe's in all twelve cells, at $0.03$--$0.59$;
the $n$-gram baselines
that \citet{wang2025} find competitive for input harmfulness do not match the
probe on the correctness target (Appendix~\ref{app:steering-results}).
\space
\RCVSection{Related Work}

\emph{Correctness estimation} treats the model's errors as one population.
ConfidNet~\citep{corbiere2019} trains an auxiliary head to predict whether a
frozen classifier is correct and abstains when that prediction is low,
P(IK)~\citep{kadavath2022} trains a model to report the probability that it
knows the answer, and a linear probe on a judge's hidden
states~\citep{radharapu2025} returns a calibrated estimate of whether its
verdict is correct. All three leave the verdict in place.
Post-hoc deferral~\citep{narasimhan2022} estimates the base model's error
to decide when to consult an expert. RCV instead conditions on the
verdict and flips it.

\emph{Adapting a frozen model} attaches post-hoc control: guardrail
stacking (NeMo~\citep{rebedea2023}) adds rails and
constitution-trained classifiers~\citep{sharma2025} fine-tune new ones.
Low-cost activation monitors~\citep{cunningham2025} train linear probes on
the policy model's own activations, reusing its representations in place of a
standalone safety classifier; RCV's probe reads the frozen
classifier's internal state and predicts when its verdict is wrong.
Fair-Wrapping~\citep{soen2022}
edits a fixed model's output posteriors toward a fairness target; RCV flips
verdicts toward a safety policy.
RepV~\citep{yang2025}, the closest architectural antecedent, learns a single
shared space for interpreter correctness in a different setting, plan
verification.

\emph{Drift monitoring} is mature. Label-free performance estimation
(ATC~\citep{garg2022}) estimates accuracy under shift; sequential risk
tracking~\citep{podkopaev2022} detects increases in risk from observed
losses. Its label-free extension~\citep{amoukou2024} uses a trained error
estimator to monitor the proportion of high-error observations.
DriftLens~\citep{greco2024}
localizes per-label representation drift, the Suitability
Filter~\citep{pouget2025} gates dataset-level accuracy drops, JailbreaksOverTime~\citep{piet2025}
self-trains detectors as jailbreaks drift, and WATCH~\citep{prinster2025}
raises anytime-valid alarms and adapts to benign shift.
Concurrent work,
DriftGuard~\citep{xin2026}, monitors moderation drift and \emph{updates} the
moderator's weights; RCV corrects verdicts without modifying the weights, and fine-tunes them only
as its last resort.

\emph{Verdict-conditioned estimation} has precedents. Top-label
calibration~\citep{gupta2022} is the named prior for per-verdict calibration,
and class-conditional conformal prediction~\citep{ding2023} splits by the true class for
set-valued coverage. RCV splits by the verdict. The one correctness estimate
it returns underlies both decision correction at inference and drift
monitoring in deployment.
\RCVSection{Limitations}

First, steering is measured on all six classifier--dataset cells;
detection is limited to one classifier, Llama-Guard-3, over ten held-out
harm families; and the end-to-end maintenance loop runs on one deployment,
WildGuardMix gated by the same classifier. The loop composes those
separately measured components; its audit, acceptance gate, and
repair cycle are exercised only in that deployment. Both
datasets are English and single-turn, and we evaluate policy coverage using
a single added clause. Second, adherence is agreement with
$Y^\star$, so every gain is relative to the deployer's stated policy as one
judge applies it. Any single reference, human or model, carries bias, and
$Y^\star$ does not enter the two arms symmetrically: RCV's probes are fit
and calibrated on it, while the raw classifier makes no use of it. The
re-scoring against the datasets' human labels, unused in RCV's training and
calibration, quantifies sensitivity to the evaluation reference
(Appendix~\ref{app:steering-results}); a multi-judge panel is the
natural extension.
Third, the census and chain evidence run on composed streams: drift is
injected into recorded base traffic and audits draw from a finite corpus.
A recorded per-episode verdict can flip when the evaluation set is
enlarged, as it does for the fine-tuned episode.
Fourth, drift is simulated rather than naturally
occurring:
each campaign is a category from the dataset's own taxonomy, held
out of initial probe training and injected into base traffic on a fixed dose
schedule, one campaign at a time, so composition and onset are
fixed by construction. Drift as it arises in live
deployment is untested. Fifth, the
probes require white-box access. An API-only deployment exposes at most a
decision score, and the score-only probe recovers much of the adherence but
far less of the caught share, since missed unsafe items sit behind
confident safe verdicts whose scores carry little signal.
\RCVSection{Conclusion}
A deployer's safety policy specifies which prompt--response pairs a
language model application should allow and which it should block. The
application assembles the prompt from the user's message and any additional
context, and its language model generates the reply. An off-the-shelf safety
classifier then labels the pair safe or unsafe, but what it learned to flag
during training may differ from what the deployer prohibits.
Regime-Conditional Verification (RCV) addresses this mismatch through
correctness estimation. From the classifier's internal state, it predicts
how likely each verdict is to disagree with the deployer's policy. RCV learns
from examples labeled under that policy, using one probe for safe verdicts and
another for unsafe verdicts. It uses these estimates to reverse verdicts
likely to be wrong. As prompt--response pairs arrive over time, RCV uses the same
estimates as a label-free signal to monitor this traffic for drift.

Across two datasets of English, single-turn exchanges, RCV improved all three
classifiers' adherence to the deployer's policy without retraining them, and
mean agreement with the datasets' human annotations also improved in four of
the six classifier--dataset combinations. One clause of the deployer's policy
treated an assistant's agreement to help with a harmful request as a violation
even when the reply gave no harmful instructions. Two of the classifiers had
never been trained to enforce that rule. On PKU-SafeRLHF, RCV raised recall on
the added rule for both classifiers to nearly the level each reached on rules
covered in training.

For maintenance, RCV monitors safe and unsafe verdicts separately. For each monitoring threshold, a one-sided cumulative sum test accumulates
evidence that correctness estimates are below it more often than in drift-free
traffic. When any test reaches its calibrated alarm threshold,
RCV triggers a labeled audit. In the detection experiment, we introduced
examples from one held-out harm category at a time
into recorded traffic. The monitor raised an alarm in every run of that
experiment. In a separate repair experiment, we responded to each alarm by
auditing a fresh sample of prompt--response pairs and using their labels to
refit the probes. In most drift episodes, the updated system met the
deployer's standards for catching violations and limiting unnecessary blocks
within the labeling budget. RCV flagged the remaining episodes for classifier
fine-tuning.

A practical limitation of RCV is that its probes require access to the
classifier's internal representations. A probe using only the decision score
recovered a smaller share of the violations the classifier had missed. We
tested monitoring and repair on one classifier across ten harm categories held
out of probe training. We simulated drift by starting each run with baseline
traffic and then linearly increasing the probability of drawing a
prompt--response pair from the held-out category.

Taken together, these results show how deployers can use correctness
estimation to steer a classifier's decisions toward their own policies and
keep it in service by updating the probes. As more cases are reviewed under
the deployer's policy, RCV can learn from their labels to correct similar
errors in later traffic. A deployer considering a replacement for its
classifier could test each candidate with RCV on the same held-out pairs and
measure how often the resulting verdicts agree with the reference judgments.
With RCV, a deployer can use the same written policy to guide corrections to
individual verdicts, test probe updates after a drift alarm, and compare
classifiers when considering a replacement.
\section*{Ethical Statement}
This work builds and evaluates on uncensored, potentially harmful
prompt--response content. We redistribute neither the datasets nor any harmful
text: rebuilder scripts reconstruct both evaluation sets from the official
releases under their licenses and access conditions, and the
oracle's labels ship as a label map keyed to public item identifiers.
Verdict correction can over-block legitimate speech or pass content the
classifier had caught; both directions sit under deployer-set controls.
\bibliography{references}

\begin{thebibliography}{26}
\providecommand{\natexlab}[1]{#1}

\bibitem[{Amoukou et~al.(2024)Amoukou, Bewley, Mishra, L{\'e}cu{\'e},
  Magazzeni, and Veloso}]{amoukou2024}
Amoukou, S.~I.; Bewley, T.; Mishra, S.; L{\'e}cu{\'e}, F.; Magazzeni, D.; and
  Veloso, M. 2024.
\newblock Sequential Harmful Shift Detection Without Labels.
\newblock In \emph{Advances in Neural Information Processing Systems 37}.

\bibitem[{Clymer et~al.(2024)Clymer, Gabrieli, Krueger, and
  Larsen}]{clymer2024}
Clymer, J.; Gabrieli, N.; Krueger, D.; and Larsen, T. 2024.
\newblock Safety Cases: How to Justify the Safety of Advanced {AI} Systems.
\newblock \emph{arXiv preprint arXiv:2403.10462}.

\bibitem[{Corbi{\`e}re et~al.(2019)Corbi{\`e}re, Thome, Bar-Hen, Cord, and
  P{\'e}rez}]{corbiere2019}
Corbi{\`e}re, C.; Thome, N.; Bar-Hen, A.; Cord, M.; and P{\'e}rez, P. 2019.
\newblock Addressing Failure Prediction by Learning Model Confidence.
\newblock In \emph{Advances in Neural Information Processing Systems 32}.

\bibitem[{Cunningham et~al.(2025)Cunningham, Peng, Wei, Ong, Roger, Petrini,
  Wagner, Mikulik, and Sharma}]{cunningham2025}
Cunningham, H.; Peng, A.; Wei, J.; Ong, E.; Roger, F.; Petrini, L.; Wagner, M.;
  Mikulik, V.; and Sharma, M. 2025.
\newblock Cost-Effective Constitutional Classifiers via Representation Re-use.
\newblock Anthropic Alignment Science Blog,
  \url{https://alignment.anthropic.com/2025/cheap-monitors/}.

\bibitem[{Ding et~al.(2023)Ding, Angelopoulos, Bates, Jordan, and
  Tibshirani}]{ding2023}
Ding, T.; Angelopoulos, A.~N.; Bates, S.; Jordan, M.~I.; and Tibshirani, R.~J.
  2023.
\newblock Class-Conditional Conformal Prediction with Many Classes.
\newblock In \emph{Advances in Neural Information Processing Systems 36}.

\bibitem[{Gama et~al.(2014)Gama, {\v{Z}}liobait{\.e}, Bifet, Pechenizkiy, and
  Bouchachia}]{gama2014}
Gama, J.; {\v{Z}}liobait{\.e}, I.; Bifet, A.; Pechenizkiy, M.; and Bouchachia,
  A. 2014.
\newblock A Survey on Concept Drift Adaptation.
\newblock \emph{ACM Computing Surveys}, 46(4): 44:1--44:37.

\bibitem[{Garg et~al.(2022)Garg, Balakrishnan, Lipton, Neyshabur, and
  Sedghi}]{garg2022}
Garg, S.; Balakrishnan, S.; Lipton, Z.~C.; Neyshabur, B.; and Sedghi, H. 2022.
\newblock Leveraging Unlabeled Data to Predict Out-of-Distribution Performance.
\newblock In \emph{International Conference on Learning Representations}.

\bibitem[{Grattafiori et~al.(2024)}]{grattafiori2024}
Grattafiori, A.; et~al. 2024.
\newblock The {Llama} 3 Herd of Models.
\newblock \emph{arXiv preprint arXiv:2407.21783}.

\bibitem[{Greco et~al.(2025)Greco, Vacchetti, Apiletti, and
  Cerquitelli}]{greco2024}
Greco, S.; Vacchetti, B.; Apiletti, D.; and Cerquitelli, T. 2025.
\newblock Unsupervised Concept Drift Detection from Deep Learning
  Representations in Real-time.
\newblock \emph{IEEE Transactions on Knowledge and Data Engineering}.

\bibitem[{Gupta and Ramdas(2022)}]{gupta2022}
Gupta, C.; and Ramdas, A. 2022.
\newblock Top-Label Calibration and Multiclass-to-Binary Reductions.
\newblock In \emph{International Conference on Learning Representations}.

\bibitem[{Han et~al.(2024)Han, Rao, Ettinger, Jiang, Lin, Lambert, Choi, and
  Dziri}]{han2024}
Han, S.; Rao, K.; Ettinger, A.; Jiang, L.; Lin, B.~Y.; Lambert, N.; Choi, Y.;
  and Dziri, N. 2024.
\newblock {WildGuard}: Open One-Stop Moderation Tools for Safety Risks,
  Jailbreaks, and Refusals of {LLMs}.
\newblock \emph{arXiv preprint arXiv:2406.18495}.

\bibitem[{Ji et~al.(2023)Ji, Liu, Dai, Pan, Zhang, Bian et~al.}]{ji2023}
Ji, J.; Liu, M.; Dai, J.; Pan, X.; Zhang, C.; Bian, C.; et~al. 2023.
\newblock {BeaverTails}: Towards Improved Safety Alignment of {LLM} via a
  Human-Preference Dataset.
\newblock \emph{arXiv preprint arXiv:2307.04657}.

\bibitem[{Kadavath et~al.(2022)Kadavath, Conerly, Askell, Henighan, Drain,
  Perez, Schiefer, Hatfield-Dodds, DasSarma, Tran-Johnson
  et~al.}]{kadavath2022}
Kadavath, S.; Conerly, T.; Askell, A.; Henighan, T.; Drain, D.; Perez, E.;
  Schiefer, N.; Hatfield-Dodds, Z.; DasSarma, N.; Tran-Johnson, E.; et~al.
  2022.
\newblock Language Models (Mostly) Know What They Know.
\newblock \emph{arXiv preprint arXiv:2207.05221}.

\bibitem[{Kivim{\"a}ki et~al.(2025)Kivim{\"a}ki, Nurminen, Bia{\l}ek, and
  Kuberski}]{kivimaki2025}
Kivim{\"a}ki, J.; Nurminen, J.~K.; Bia{\l}ek, J.; and Kuberski, W. 2025.
\newblock Confidence-based Estimators for Predictive Performance in Model
  Monitoring.
\newblock \emph{Journal of Artificial Intelligence Research}, 82: 209--240.

\bibitem[{Narasimhan et~al.(2022)Narasimhan, Jitkrittum, Menon, Rawat, and
  Kumar}]{narasimhan2022}
Narasimhan, H.; Jitkrittum, W.; Menon, A.~K.; Rawat, A.; and Kumar, S. 2022.
\newblock Post-hoc Estimators for Learning to Defer to an Expert.
\newblock In \emph{Advances in Neural Information Processing Systems 35}.

\bibitem[{Piet et~al.(2025)Piet, Huang, Jacob, Chow, Alrashed, Zhao, Hu,
  Sitawarin, Alomair, and Wagner}]{piet2025}
Piet, J.; Huang, X.; Jacob, D.; Chow, A.; Alrashed, M.; Zhao, G.; Hu, Z.;
  Sitawarin, C.; Alomair, B.; and Wagner, D. 2025.
\newblock {JailbreaksOverTime}: Detecting Jailbreak Attacks Under Distribution
  Shift.
\newblock \emph{arXiv preprint arXiv:2504.19440}.

\bibitem[{Podkopaev and Ramdas(2022)}]{podkopaev2022}
Podkopaev, A.; and Ramdas, A. 2022.
\newblock Tracking the Risk of a Deployed Model and Detecting Harmful
  Distribution Shifts.
\newblock In \emph{International Conference on Learning Representations}.

\bibitem[{Pouget et~al.(2025)Pouget, Yaghini, Rabanser, and
  Papernot}]{pouget2025}
Pouget, A.; Yaghini, M.; Rabanser, S.; and Papernot, N. 2025.
\newblock Suitability Filter: A Statistical Framework for Classifier Evaluation
  in Real-World Deployment Settings.
\newblock In \emph{Proceedings of the 42nd International Conference on Machine
  Learning}.

\bibitem[{Prinster et~al.(2025)Prinster, Han, Liu, and Saria}]{prinster2025}
Prinster, D.; Han, X.; Liu, A.; and Saria, S. 2025.
\newblock {WATCH}: Adaptive Monitoring for {AI} Deployments via
  Weighted-Conformal Martingales.
\newblock In \emph{Proceedings of the 42nd International Conference on Machine
  Learning}.

\bibitem[{Radharapu et~al.(2025)Radharapu, Saxena, Li, Whitehouse, Williams,
  and Cancedda}]{radharapu2025}
Radharapu, B.; Saxena, E.; Li, K.; Whitehouse, C.; Williams, A.; and Cancedda,
  N. 2025.
\newblock Calibrating {LLM} Judges: Linear Probes for Fast and Reliable
  Uncertainty Estimation.
\newblock \emph{arXiv preprint arXiv:2512.22245}.

\bibitem[{Rebedea et~al.(2023)Rebedea, Dinu, Sreedhar, Parisien, and
  Cohen}]{rebedea2023}
Rebedea, T.; Dinu, R.; Sreedhar, M.; Parisien, C.; and Cohen, J. 2023.
\newblock {NeMo} Guardrails: A Toolkit for Controllable and Safe {LLM}
  Applications with Programmable Rails.
\newblock In \emph{Proceedings of the 2023 Conference on Empirical Methods in
  Natural Language Processing: System Demonstrations}, 431--445.

\bibitem[{Sharma et~al.(2025)Sharma, Tong, Mu, Wei, Kruthoff, Goodfriend, Ong,
  Peng, Agarwal, Anil et~al.}]{sharma2025}
Sharma, M.; Tong, M.; Mu, J.; Wei, J.; Kruthoff, J.; Goodfriend, S.; Ong, E.;
  Peng, A.; Agarwal, R.; Anil, C.; et~al. 2025.
\newblock Constitutional Classifiers: Defending against Universal Jailbreaks
  across Thousands of Hours of Red Teaming.
\newblock \emph{arXiv preprint arXiv:2501.18837}.

\bibitem[{Soen et~al.(2022)Soen, Alabdulmohsin, Koyejo, Mansour, Moorosi, Nock,
  Sun, and Xie}]{soen2022}
Soen, A.; Alabdulmohsin, I.; Koyejo, S.; Mansour, Y.; Moorosi, N.; Nock, R.;
  Sun, K.; and Xie, L. 2022.
\newblock Fair Wrapping for Black-box Predictions.
\newblock In \emph{Advances in Neural Information Processing Systems 35}.

\bibitem[{Wang et~al.(2026)Wang, Wei, Liu, Zhou, and Chen}]{wang2025}
Wang, C.; Wei, Z.; Liu, Q.; Zhou, W.; and Chen, M. 2026.
\newblock False Sense of Security: Why Probing-based Malicious Input Detection
  Fails to Generalize.
\newblock In \emph{Findings of the Association for Computational Linguistics:
  ACL 2026}, 26100--26113.

\bibitem[{Xin et~al.(2026)Xin, Cai, Shen, Jin, and Hu}]{xin2026}
Xin, Y.; Cai, H.; Shen, B.; Jin, L.; and Hu, L. 2026.
\newblock {DriftGuard}: Safety-Aware Multi-Monitor Detection and Selective
  Adaptation for Evolving Toxicity Moderation.
\newblock \emph{arXiv preprint arXiv:2606.28725}.

\bibitem[{Yang et~al.(2025)Yang, Bhatt, Samineni, Siva, Wang, and
  Topcu}]{yang2025}
Yang, Y.; Bhatt, N.~P.; Samineni, P.; Siva, R.; Wang, Z.; and Topcu, U. 2025.
\newblock {RepV}: Safety-Separable Latent Spaces for Scalable Neurosymbolic
  Plan Verification.
\newblock \emph{arXiv preprint arXiv:2510.26935}.

\end{thebibliography}

\RCVAppendixBegin
\RCVAppendixChapter{Formal Propositions}\label{app:props}

\RCVAppendixSection{Calibration Across Verdict Regimes (Proposition~1)}\label{app:steering}

Suppose the items in both the \textsc{safe} and \textsc{unsafe} regimes are
scored by one probe followed by one calibration map; write $U$ for the probe's score before calibration. The ideal single
map assigns to each score value the conditional probability of a correct
verdict given that score, $m(u)=\Pr[A{=}1\mid U{=}u]$. At a fixed score $u$,
write $w(u)$ for the \textsc{safe} regime's share of the items with that
score, so the \textsc{unsafe} regime's share is $1-w(u)$, and write
$q_{\mathrm{s}}(u)$ and $q_{\mathrm{u}}(u)$ for the probability of a correct
verdict at that score within the \textsc{safe} and \textsc{unsafe} regimes, at scores where both
regimes have items. These last two are what
a map calibrated to its own regime would report at that score, so
$m$, $q_{\mathrm{s}}$ and $q_{\mathrm{u}}$ are the same kind of object: one
map for the two regimes together and one for each regime alone.

\begin{proposition}\label{prop:misalignment}
For almost every score $u$, the ideal single map is the average of the two
regimes' conditional agreement rates, weighted by the regimes' shares,
\[
m(u)=w(u)\,q_{\mathrm{s}}(u)+\bigl(1-w(u)\bigr)\,q_{\mathrm{u}}(u),
\]
and its miscalibration on each regime is the other regime's share
times the gap between the two rates: on the \textsc{unsafe} regime,
$m(u)-q_{\mathrm{u}}(u)=w(u)\,\bigl(q_{\mathrm{s}}(u)-q_{\mathrm{u}}(u)\bigr)$,
and symmetrically on the \textsc{safe} regime.
\end{proposition}

\begin{proof}
At a fixed score, split the correct-verdict event by regime: the conditional
probability of a correct verdict is the average of the two regimes' agreement
rates, weighted by the regimes' shares, which is the first display. Subtracting
$q_{\mathrm{u}}(u)$ from both sides gives the \textsc{unsafe}-regime form, and
the regimes' roles are symmetric.
\end{proof}

Two consequences read directly off the identity. The ideal single map is
miscalibrated on both regimes wherever the two rates differ: it is between
them. Where the two rates differ and the regimes' shares are unequal, the
miscalibration is larger on the minority regime. Where one regime dominates a
score region ($w(u)$ near one), the map is nearly exact for that regime, and
its miscalibration on the minority regime is nearly the full gap. The identity does
not constrain the magnitude of the gap.
The gap vanishes at almost every score if and
only if correctness is conditionally independent of the verdict given the score.
Either the verdict adds no information about correctness beyond the score, or
no single map matches both regimes' conditional agreement rates at each
underlying score.

\paragraph{Probe routing.}
Proposition~1 concerns calibration of a fixed score. Separate maps can
assign different probabilities to the same score, but a shared linear
probe still uses one feature combination for both regimes.

With logistic calibration, its outputs have the form
\[
\sigma(a_y v^\top Z+b_y), \qquad y\in\{0,1\},
\]
where $\sigma$ is the logistic function, $v$ is the shared weight vector,
$a_y$ is a calibration slope, and $b_y$ is an intercept.
The two coefficient vectors are therefore multiples of $v$.

Given that $Z$ has at least two coordinates, choose two non-collinear
weight vectors $r$ and $s$. Use these weights for the two routed probes,
set their intercepts to zero, and use $\sigma$ as each calibration map.
The resulting functions are $\sigma(r^\top Z)$ and $\sigma(s^\top Z)$.
Suppose a shared probe could represent both functions for every $Z$.
Since $\sigma$ is injective, this would require $a_0v=r$ and $a_1v=s$.
Thus $r$ and $s$ would both be multiples of $v$, contradicting our
choice of non-collinear vectors.

\RCVAppendixSection{Agreement Rates from Calibrated Probe Estimates (Proposition~2)}\label{app:monitor}

Proposition~2 states the identity behind confidence-based performance
estimation (CBPE)~\citep{kivimaki2025} one predicted class at a time; the
verdict regime is the predicted class. Average-confidence estimators rest on the
same averaging step, and ATC~\citep{garg2022} applies a threshold to the same
confidence signal. When the probe is calibrated to current traffic, its mean
score over a verdict regime is a plug-in estimate of that regime's
agreement rate.

Fix a regime $y\in\{0,1\}$ with $\Pr[Y{=}y]>0$ and consider the items the
classifier routes there. Applying the regime's calibration map to the probe's logistic-regression
output gives $S\in[0,1]$. This calibrated score is the probe's estimate of
$\Pr[A{=}1\mid Z,\,Y{=}y]$, the probability that the verdict is correct.

\begin{proposition}\label{prop:monitor}
Suppose the score is calibrated to current traffic within the regime: with all
probabilities taken over the current traffic distribution,
\begin{equation}\label{eq:calibration}
\Pr[A=1 \mid S,\, Y=y] \;=\; S \quad\text{almost surely.}
\end{equation}
Then
\[
\mathbb{E}[S \mid Y=y] \;=\; \Pr[A=1 \mid Y=y],
\]
the regime's agreement rate on current traffic.
\end{proposition}

\begin{proof}
Since $A$ is binary, condition~\eqref{eq:calibration} states
$S=\mathbb{E}[\mathbf{1}\{A{=}1\}\mid S,\,Y{=}y]$ almost surely. Taking
expectations conditional on $Y{=}y$,
\begin{align*}
\mathbb{E}[S\mid Y{=}y]
&=\mathbb{E}\bigl[\mathbb{E}[\mathbf{1}\{A{=}1\}\mid S,\,Y{=}y]\;\big|\;Y{=}y\bigr]\\
&=\Pr[A{=}1\mid Y{=}y].
\end{align*}
The second equality holds because averaging the agreement rate at each score
over the regime's score distribution returns the regime's overall agreement
rate.
\end{proof}

The regime agreement rates are the components of adherence: adherence is
their traffic-weighted average, and the identity concerns one regime
at a time. Proposition~2's estimator is the window mean: over a window drawn
uniformly from the regime's traffic, the average score is an unbiased estimate
of the regime's mean score, and therefore of the regime's agreement rate whenever
condition~\eqref{eq:calibration} holds on that traffic. Drift erodes that
condition: a score whose map was fitted on one window's labeled data
need not
stay calibrated after the traffic shifts. The loop refits the probe on every
alarm, on a block composed at the alarm-time mixture and therefore
representative of the post-alarm traffic by construction. Between refits, calibration is an
assumption; the identity certifies the estimator within calibration, not
through drift.

\RCVAppendixChapter{Steering: Extended Results}\label{app:steering-results}

\RCVAppendixSection{Human-Label Evaluation Grid}\label{app:eval-swap}

We re-score the six classifier--dataset cells against the datasets' human
labels in place of $Y^\star$, with no refit and no threshold change.
Corrections are judged by agreement with those labels and not only with
$Y^\star$; the calibrated correctness estimates and corrected verdicts are functions of the
representation and of the per-regime probe, which was trained on $Y^\star$'s
labels. A correction that matched
$Y^\star$ on every scored item would score, on the human axis, the raw
agreement between $Y^\star$ and the human labels: $0.9195$ on PKU-SafeRLHF and
$0.9298$ on WildGuardMix ($\kappa=0.838$ and $0.766$). The movement count below
measures distance to this agreement level, computed on each seed's own evaluation
split, where it varies by $\pm0.2$--$0.6$\,pp. Corrected adherence moves toward the level in $48$ of the
$50$ seeds of the five cells that start a measurable distance from it (cell
means $1.8$--$10.5$\,pp), and the distance shrinks under correction in each
(paired $p\le 0.042$;
Table~\ref{tab:eval-swap}). The sixth cell, WildGuard on PKU-SafeRLHF, starts
within $0.7$\,pp of the level on every seed and ends within $1.4$\,pp on every
seed; its ten-seed mean distance to the level is $0.4$\,pp before correction and
$0.7$\,pp after, and the mean crosses from below the level to above it.
The flips do not concentrate on the items where $Y^\star$ diverges from
the annotators: in every cell, such items are a smaller share of the
flipped items than of all verdict--$Y^\star$ disagreements, by $6.5$ to
$29.1$\,pp.
On the human axis, the caught-share changes of the two generative
classifiers on WildGuardMix (Table~\ref{tab:eval-swap}), whose caught
populations average $70$ and $24$ items per seed, are not distinguishable from seed noise
($p=0.10$, $0.99$).

\emph{Derivation.} A fixed script re-derives every cell from the
per-seed steering records; if any $Y^\star$-axis value fails to reproduce
Table~\ref{tab:t1v3-steering} of the main text to $|\Delta|=0$, it aborts and does not emit the grid.
Every paired $p$ in this subsection comes from a two-sided paired $t$-test
across the
ten evaluation seeds.
The level is a function of the two label sets alone and is fixed before
any correction runs; the movement count was specified after the per-cell
directions of movement (four up, two down; Table~\ref{tab:eval-swap})
were observed.

\begin{table*}[t]
\centering\small
\setlength{\tabcolsep}{6pt}
\begin{tabular}{@{}lll@{}}
\toprule
Classifier & Caught share ($Y^\star{\to}$human) & Adherence, human axis (raw$\to$RCV) \\
\midrule
\multicolumn{3}{@{}l}{\emph{PKU-SafeRLHF}} \\
Llama-Guard-3 & $0.703\to0.754$ $(+5.1)$ & $0.892\to0.922$ $(+3.1)$ \\
WildGuard & $0.383\to0.434$ $(+5.1)$ & $0.917\to0.923$ $(+0.7)$ \\
Beaver & $0.806\to0.845$ $(+3.8)$ & $0.813\to0.916$ $(+10.3)$ \\
\multicolumn{3}{@{}l}{\emph{WildGuardMix}} \\
Llama-Guard-3 & $0.346\to0.374$ $(+2.8)$ & $0.950\to0.946$ $(-0.3)$ \\
WildGuard & $0.291\to0.291$ $(-0.0)$ & $0.987\to0.957$ $(-2.9)$ \\
Beaver & $0.546\to0.568$ $(+2.3)$ & $0.903\to0.938$ $(+3.5)$ \\
\bottomrule
\end{tabular}
\caption{\textbf{Adherence against the datasets' human labels moves toward
$Y^\star$'s own agreement with them ($0.9195$ PKU-SafeRLHF / $0.9298$
WildGuardMix) in the five cells with a measurable gap: up from below, down
from above.} Each row is one deployed steering
cell of Table~\ref{tab:t1v3-steering} of the main text (the classifier named in the Classifier
column). That cell is held fixed
and re-scored against the dataset's own human annotations.
Caught share (fraction of truly-unsafe passed items RCV flips to
\textsc{unsafe}): the $Y^\star$ value followed by the human-axis value, with the
change in percentage points. Adherence: raw$\to$RCV on the human axis.
Ten-seed means.}\label{tab:eval-swap}
\end{table*}

\RCVAppendixSection{Steering Toward the Added Clause}\label{app:alien}

RCV increases recall on the added clause in all four classifier--dataset
settings (Fig.~\ref{fig:clause}). We compare the share of misses recovered for
each group, subtracting the in-training share from the added-clause share.
WildGuard is excluded because the added clause is already in its training.
At $\tau=0.5$, we report average differences across ten random splits
of each fixed dataset, with approximate $95\%$ confidence intervals. The
operating-point counts sweep $\tau=0.02,0.04,\ldots,0.94$ with the fitted
probes fixed.

On PKU-SafeRLHF, RCV recovers a greater share of Llama-Guard-3's missed
violations of the added clause than of the in-training clauses
($+0.0711$, $95\%$ CI $[+0.052, +0.090]$). The mean advantage remains
positive at all $47$ tested operating points, including the deployed
threshold. The corresponding difference for Beaver is $+0.0090$
($95\%$ CI $[-0.017, +0.035]$). On WildGuardMix, the differences are
$+0.020$ for Beaver ($95\%$ CI $[-0.022, +0.062]$) and $-0.089$ for
Llama-Guard-3 ($95\%$ CI $[-0.187, +0.009]$), with positive mean
differences at $20$ and $3$ of the $47$ operating points, respectively.

We hypothesize that the added-clause recovery draws on features learned
during broad pretraining and retained through safety fine-tuning. Given
examples labeled under the deployer's policy, the probe can use these
features to recover violations that the classifier misses.

\RCVAppendixSection{Routing Ablations}\label{app:probe-family}

Per-regime routing exceeds its pooled counterpart in mean AUROC in every
cell of the PKU-SafeRLHF routing ablation, by $+0.2$ to
$+4.6$\,pp, on a linear probe, on a multilayer perceptron, and on
gradient-boosted trees. The three families differ in capacity, yet
routing improves AUROC in all three.

\paragraph{Probe and calibration routing.}
Table~\ref{tab:routing-components} reports the four combinations of
shared and per-verdict probe fitting and calibration. With routed probes,
separate calibration increases adherence in 22 paired runs, leaves it
unchanged in 15, and decreases it in 23.
Probe routing alone recovers essentially the full adherence gain over pooling.
\begin{table*}[t]
\centering
\small
\begin{tabular}{@{}llcccc@{}}
\toprule
Probes & Calibration & Adherence $\uparrow$ & Caught share $\uparrow$ &
ECE $\downarrow$ & AUROC $\uparrow$ \\
\midrule
\multicolumn{6}{l}{\textbf{PKU-SafeRLHF: Llama-Guard-3}} \\
Shared & Shared & $0.915\pm0.003$ & $0.660\pm0.022$ & $0.014\pm0.002$ & $0.939\pm0.004$ \\
Shared & Per verdict & $0.918\pm0.003$ & $0.691\pm0.019$ & $0.015\pm0.002$ & $0.939\pm0.005$ \\
Per verdict & Shared & $0.926\pm0.004$ & $0.704\pm0.022$ & $0.013\pm0.002$ & $0.949\pm0.006$ \\
Per verdict & Per verdict & $0.926\pm0.004$ & $0.703\pm0.021$ & $0.013\pm0.002$ & $0.950\pm0.006$ \\
\midrule
\multicolumn{6}{l}{\textbf{PKU-SafeRLHF: WildGuard}} \\
Shared & Shared & $0.912\pm0.007$ & $0.240\pm0.026$ & $0.015\pm0.002$ & $0.913\pm0.009$ \\
Shared & Per verdict & $0.913\pm0.007$ & $0.172\pm0.036$ & $0.016\pm0.003$ & $0.914\pm0.009$ \\
Per verdict & Shared & $0.931\pm0.005$ & $0.381\pm0.046$ & $0.014\pm0.003$ & $0.940\pm0.006$ \\
Per verdict & Per verdict & $0.932\pm0.005$ & $0.383\pm0.063$ & $0.014\pm0.002$ & $0.941\pm0.006$ \\
\midrule
\multicolumn{6}{l}{\textbf{PKU-SafeRLHF: Beaver}} \\
Shared & Shared & $0.867\pm0.006$ & $0.627\pm0.021$ & $0.018\pm0.003$ & $0.914\pm0.004$ \\
Shared & Per verdict & $0.889\pm0.005$ & $0.733\pm0.019$ & $0.017\pm0.003$ & $0.935\pm0.004$ \\
Per verdict & Shared & $0.920\pm0.002$ & $0.807\pm0.020$ & $0.013\pm0.003$ & $0.961\pm0.003$ \\
Per verdict & Per verdict & $0.920\pm0.002$ & $0.806\pm0.019$ & $0.013\pm0.003$ & $0.961\pm0.003$ \\
\midrule
\multicolumn{6}{l}{\textbf{WildGuardMix: Llama-Guard-3}} \\
Shared & Shared & $0.935\pm0.003$ & $0.264\pm0.055$ & $0.012\pm0.003$ & $0.934\pm0.007$ \\
Shared & Per verdict & $0.935\pm0.004$ & $0.254\pm0.052$ & $0.013\pm0.004$ & $0.934\pm0.007$ \\
Per verdict & Shared & $0.943\pm0.003$ & $0.335\pm0.039$ & $0.012\pm0.002$ & $0.943\pm0.006$ \\
Per verdict & Per verdict & $0.943\pm0.004$ & $0.346\pm0.047$ & $0.012\pm0.003$ & $0.944\pm0.006$ \\
\midrule
\multicolumn{6}{l}{\textbf{WildGuardMix: WildGuard}} \\
Shared & Shared & $0.936\pm0.002$ & $0.186\pm0.030$ & $0.012\pm0.002$ & $0.923\pm0.006$ \\
Shared & Per verdict & $0.936\pm0.002$ & $0.194\pm0.033$ & $0.011\pm0.002$ & $0.923\pm0.006$ \\
Per verdict & Shared & $0.945\pm0.002$ & $0.275\pm0.050$ & $0.010\pm0.003$ & $0.935\pm0.007$ \\
Per verdict & Per verdict & $0.945\pm0.003$ & $0.291\pm0.034$ & $0.011\pm0.003$ & $0.936\pm0.005$ \\
\midrule
\multicolumn{6}{l}{\textbf{WildGuardMix: Beaver}} \\
Shared & Shared & $0.903\pm0.004$ & $0.170\pm0.026$ & $0.013\pm0.002$ & $0.893\pm0.008$ \\
Shared & Per verdict & $0.914\pm0.006$ & $0.361\pm0.040$ & $0.014\pm0.003$ & $0.918\pm0.006$ \\
Per verdict & Shared & $0.933\pm0.004$ & $0.545\pm0.029$ & $0.014\pm0.003$ & $0.936\pm0.008$ \\
Per verdict & Per verdict & $0.933\pm0.003$ & $0.546\pm0.029$ & $0.013\pm0.003$ & $0.938\pm0.007$ \\
\bottomrule
\end{tabular}
\caption{\textbf{Separate probe and calibration routing.}
Entries are means $\pm$ standard deviations over ten matched seeds,
scored against the policy judge. Shared denotes one fit for both verdict
regimes; per verdict denotes a separate fit for each regime. Only the two
routing settings vary. All four metrics are reported on a scale from
zero to one.}
\label{tab:routing-components}
\end{table*}

\RCVAppendixSection{The Decision-Score Probe}\label{app:score-probe}

A probe routed by verdict and using the classifier's scalar decision score
in place of its internal representation $Z$ recovers adherence to within
$5.1$ points of the deployed probe but catches fewer of the missed-unsafe
items (Table~\ref{tab:score-probe}).
We use the continuous harm cost for Beaver and the unsafe-token probability
for the two generative classifiers. The probe is fitted and its
output calibrated under the same steering protocol
(Appendix~\ref{app:recipes}). The score enters the probe through a single
monotone standardization, which preserves its rank ordering, so the internal
state's higher caught share below is not an
artifact of the transform applied to the score.

The main text reports twelve input-ablation evaluations: we score each of
the six classifier--dataset combinations against both $Y^\star$ and human
labels. The comparisons here and in Appendix~\ref{app:text-surface}
use $Y^\star$.
The routed score probe's caught share of the missed-unsafe items is lower in all
six cells, by $10.1$ to $47.9$ points.

\begin{table*}[t]
\centering\small
\setlength{\tabcolsep}{6pt}
\begin{tabular}{@{}l cc@{}}
\toprule
Classifier & Caught share & Adherence \\
 & ($Z$\,/\,score) & ($Z$\,/\,score) \\
\midrule
\multicolumn{3}{@{}l}{\emph{PKU-SafeRLHF}} \\
Beaver & $0.806{\pm}0.019$ / $0.699{\pm}0.022$ & $0.920{\pm}0.002$ / $0.886{\pm}0.004$ \\
Llama-Guard-3 & $0.703{\pm}0.021$ / $0.224{\pm}0.026$ & $0.926{\pm}0.004$ / $0.876{\pm}0.005$ \\
WildGuard & $0.383{\pm}0.063$ / $0.059{\pm}0.025$ & $0.932{\pm}0.005$ / $0.904{\pm}0.005$ \\
\multicolumn{3}{@{}l}{\emph{WildGuardMix}} \\
Beaver & $0.546{\pm}0.029$ / $0.445{\pm}0.025$ & $0.933{\pm}0.003$ / $0.923{\pm}0.004$ \\
Llama-Guard-3 & $0.346{\pm}0.047$ / $0.099{\pm}0.032$ & $0.943{\pm}0.004$ / $0.923{\pm}0.005$ \\
WildGuard & $0.291{\pm}0.034$ / $0.031{\pm}0.016$ & $0.945{\pm}0.003$ / $0.932{\pm}0.003$ \\
\bottomrule
\end{tabular}
\caption{\textbf{The internal state has the higher caught share in all six
cells.} Per-regime steering with the probe
reading the classifier's decision score, against the deployed probe reading its
internal representation $Z$. Both are routed on the verdict under the identical
per-regime protocol, with the same targets, probe family, calibration, splits,
and seeds. Each cell gives the two arms as
$Z$\,/\,score, scored against the policy oracle $Y^\star$. The
internal-state probe has the higher caught share in all six cells. The score
arm's adherence is at most $5.1$ points below the internal state's.
Ten-seed means $\pm$ s.d.; the caught share is the
fraction of truly-unsafe passed items RCV flips to
\textsc{unsafe}.}\label{tab:score-probe}
\end{table*}

\RCVAppendixSection{The Text-Surface Control}\label{app:text-surface}
A published critique of hidden-state probing holds that probes for malicious-input detection
learn superficial patterns and that simple $n$-gram baselines perform
comparably~\citep{wang2025}. Our probes target a different quantity, the classifier's agreement
with the deployer's oracle, so we adapt the critique's $n$-gram surface features to that
target and do not reproduce its input-harmfulness result. TF-IDF features of the prompt--response text
(word 1--2-grams and word-boundary-padded character 3--5-grams, $50{,}000$ each, sublinear term
frequencies, fit per cell and seed on the training split only) enter the identical per-regime
protocol, with the same targets, splits, seeds, calibration, and flip threshold as the deployed
probe (the sparse features enter unscaled: centring a sparse matrix makes it dense, and the rows
are already $L^2$-normalised; regularisation strength is fixed at the class default $C = 1.0$ and
is not selected per regime on a held-out fold over the deployed grid, whose largest value is $0.1$).
Logistic regression on these features has a convex objective, so the
fit is the optimum of its penalised objective. The
text side is not information-poor: the
pair's text is the classifier's own input.
The surface control's caught share of missed unsafe items,
scored against the policy oracle $Y^\star$, is lower than the
internal-state probe's in all six
classifier--dataset cells:
it reaches $0.066$--$0.558$ where the internal state reaches $0.291$--$0.806$, and
the smallest shortfall is $21.6$ points.

\RCVAppendixChapter{Maintenance: Extended Results}\label{app:maintenance-results}

This part reports extended maintenance results:
the census by family (\ref{app:census-family}), repair capacity
(\ref{app:capacity-family}), the escalation fine-tune
(\ref{app:ft-alternative}), the five chain records
(\ref{app:chain-records}), and the comparison against off-the-shelf
drift monitors under a matched false-alarm allowance
(\ref{app:battery-floor}). All five run on the WildGuardMix deployment,
its traffic classified by Llama-Guard-3.

\RCVAppendixSection{Census Outcomes by Family}\label{app:census-family}

Table~\ref{tab:census-family} breaks the $100$-episode census
by held-out family under the deployed budget ($300$ labels per attempt,
four attempts, a held-out $300$-label gate block). It reports each family's
escalation count and whether repair-in-place holds
across families. Episodes that
escalated started from lower pre-repair recall: the pre-repair probe's
median recall on the episode's own gate block is $0.706$ where the
episode escalated and $0.792$ where repair succeeded.

\begin{table}[t]
\centering\small
\setlength{\tabcolsep}{3pt}
\begin{tabular}{@{}lccccc@{}}
\toprule
Family & Repaired & First audit & Escalated & \multicolumn{2}{c}{Label cost} \\
 & /10 & /10 & /10 & median & max \\
\midrule
cyber        & 9  & 9 & 1 & 600   & 1{,}500 \\
\rowcolor{gray!10}
defamation   & 9  & 7 & 1 & 600   & 1{,}500 \\
disinformation & 6  & 4 & 4 & 1{,}200 & 1{,}500 \\
\rowcolor{gray!10}
fraud        & 10 & 9 & 0 & 600   & 900 \\
material harm & 9  & 8 & 1 & 600   & 1{,}500 \\
\rowcolor{gray!10}
mental health & 8  & 5 & 2 & 750   & 1{,}500 \\
sensitive info & 6  & 5 & 4 & 1{,}050 & 1{,}500 \\
\rowcolor{gray!10}
stereotypes  & 7  & 5 & 3 & 750   & 1{,}500 \\
toxic        & 7  & 5 & 3 & 1{,}050 & 1{,}500 \\
\rowcolor{gray!10}
violence     & 8  & 7 & 2 & 600   & 1{,}500 \\
\midrule
all          & 79 & 64 & 21 & 600 & 1{,}500 \\
\bottomrule
\end{tabular}
\caption{\textbf{In every family, most episodes repair in place; overall, most
repairs occur on the first audit.} The $100$-episode census by held-out family under
the deployed budget: episodes repaired in place, first-audit repairs, and
escalations to the fine-tune, each a count of the family's ten episodes;
we report the median and maximum label cost across all ten episodes,
including escalations. Every escalated episode costs the full $1{,}500$
labels. Ten seeds per family.}
\label{tab:census-family}
\end{table}

For the $79$ accepted updates, paired changes compare the updated and
pre-repair probes on the same post-alarm gate block. Recall improves in
$44$ episodes, is unchanged in $27$, and declines in $8$; its median
change is $+1.39$\,pp, with a range of $-4.08$ to $+9.84$\,pp.
Over-block declines in $12$, is unchanged in $42$, and rises in $25$;
its median change is $0.00$\,pp, with a range of $-1.30$ to $+0.89$\,pp.
Among accepted updates, recall improves more than five times as often as it
declines, and no increase in over-block reaches one percentage point.

\RCVAppendixSection{Repair Capacity by Family}\label{app:capacity-family}

Table~\ref{tab:capacity-family} repeats the census with the audit bounded
by data instead of by budget: the audit grows in fresh material until the
corpus is exhausted.
The two designs draw different audit segments, so an episode that narrowly
passes the gate under one design can fail it under the other.
$87$ of the $100$ episodes repair and $13$ episodes
do not pass the gate at any audit size. Under the data bound, $11$ of
the $21$ budget escalations repair and $10$ do not. These eleven episodes
can therefore be repaired with the same probe family: six first pass above
the deployed label cap, and five pass at audit sizes the budget permits,
on a different audit draw. Of the $79$ episodes that
repair under the budget, $76$ repair under the data bound and $3$ do
not.

\begin{table}[t]
\centering\small
\setlength{\tabcolsep}{1pt}
\begin{tabular}{@{}lcccc@{}}
\toprule
Family & Repaired & Above budget & Never pass & Bounds (labels) \\
 & /10 & /10 & gate /10 & of those \\
\midrule
cyber        & 10 & 1 & 0 & --- \\
\rowcolor{gray!10}
defamation   & 10 & 0 & 0 & --- \\
disinformation & 7  & 0 & 3 & 3{,}610--3{,}720 \\
\rowcolor{gray!10}
fraud        & 10 & 0 & 0 & --- \\
material harm & 10 & 1 & 0 & --- \\
\rowcolor{gray!10}
mental health & 7  & 0 & 3 & 1{,}817--2{,}245 \\
sensitive info & 10 & 3 & 0 & --- \\
\rowcolor{gray!10}
stereotypes  & 9  & 2 & 1 & 3{,}618 \\
toxic        & 7  & 0 & 3 & 3{,}479--3{,}542 \\
\rowcolor{gray!10}
violence     & 7  & 0 & 3 & 3{,}457--3{,}654 \\
\midrule
all          & 87 & 7 & 13 & 1{,}817--3{,}720 \\
\bottomrule
\end{tabular}
\caption{\textbf{$11$ of the $21$ budget escalations repair under the data
bound; the rest never pass the gate.} The same $100$ episodes with the audit bounded
by data instead of budget: episodes repaired, repairs that exceed the
deployed budget of four $300$-label attempts, episodes that do not pass
the gate at any audit size (median change in held-out recall $0.0$\,pp between
$300$ labels and the bound, range $-1.69$ to $+7.35$\,pp), and the realized data
bounds of those episodes, with a dash where a family has none. The three
mental-health episodes reach their bound at $1{,}817$--$2{,}245$ labels; the ten
others run past $3{,}457$.
Ten seeds per family.}
\label{tab:capacity-family}
\end{table}

\RCVAppendixSection{The Escalation Fine-Tune}\label{app:ft-alternative}
The main text reserves a fine-tune escalation for census episodes whose
repair does not pass the gate within the budget. We first compare the base and
fine-tuned classifiers, each with and without a refitted probe, to measure
the effect of the weight update (Table~\ref{tab:ft-cells}). We then test whether
the systems meet the acceptance standard after refitting the probes on a
fresh $300$-label audit (Table~\ref{tab:ft-rearm}).
The escalation leg runs once, on one episode of the data-bound census of
Appendix~\ref{app:capacity-family} (disinformation family, seed $8675309$),
whose thirteen audits exhausted the
$3{,}720$ fresh labels its corpus could supply while the held-out gate
reading ended at $0.7255$, the same value as at the first audit (readings vary
within $\pm4$\,pp).
The classifier is fine-tuned by LoRA on $480$ of the episode's $4{,}020$
already-acquired labels ($3{,}720$ audit labels and $300$ original gate
labels), drawn at the alarm-time mixture
($r = 0.1372$); $120$ further labels are held aside as the fine-tune's
validation set, and five trainer seeds ($42, 123, 456, 789, 1024$) repeat the
update from identical data.
The probes in cells B and D of Table~\ref{tab:ft-cells} are refit from
the deployment's probe-training and calibration slices
(Appendix~\ref{app:recipes}) plus the same $480$ items, half assigned to
fitting and half to calibration; B reads the base representation and D the
fine-tuned one. Between B and D the only changes are the classifier's
weights and the representation the probe reads.

\begin{table*}[t]
\centering\small
\setlength{\tabcolsep}{2pt}
\begin{tabular}{@{}l l l l@{}}
\toprule
Cell & System & Recall & Over-block \\
\midrule
A & base classifier & $0.7625$ & $0.0467$ \\
B & base $+$ probe (audit-half) & $0.7861$ & $0.0321$ \\
C & fine-tune, raw & $0.8254 \pm 0.0192$ & $0.0748 \pm 0.0076$ \\
D & fine-tune $+$ probe (audit-half) & $0.7914 \pm 0.0008$ & $0.0313 \pm 0.0013$ \\
\bottomrule
\end{tabular}
\caption{\textbf{The weight update raises the classifier's own recall;
the raw fine-tune also raises the over-block.} The four cells of the escalation fine-tune
on an episode that does not pass the gate at any audit size. All four
systems score the same $3{,}420$-item
held-out pool ($678$ unsafe) at the alarm-time mixture. Cells A and B do
not depend on the trainer seed; cells C and D are means $\pm$ s.d.\ over
the five trainer seeds. The fine-tuned classifier's own recall rises by
$6.3 \pm 1.9$\,pp over the base classifier (C against A). The acceptance
verdicts after the fresh $300$-label probe refit are in
Table~\ref{tab:ft-rearm}: on the powered set, the raw
fine-tune exceeds the over-block ceiling in four of five seeds; fine-tune
plus probe passes in
five of five.}\label{tab:ft-cells}
\end{table*}

After the fine-tune the probe is re-armed on a single fresh $300$-label
audit, drawn from the episode's already-obtained labels ($150$
fitting items, $150$ calibration items), and on the
probe-training and calibration slices, re-read under the fine-tuned
representation; the fine-tuned classifier plus re-armed probe is then
validated against the acceptance standard the episode
failed: the pre-drift reading, taken at deployment start (recall
$0.8105$, over-block $0.0192$), with the $5$\,pp tolerance on each axis.
The base comparator in Table~\ref{tab:ft-rearm} receives the same
re-arm construction at the base representation. The same $3{,}420$-item pool
evaluated in Table~\ref{tab:ft-cells} is now divided into the $300$-item
refitting audit, a new $300$-item gate block, and the $2{,}820$-item powered
evaluation set in Table~\ref{tab:ft-rearm}.
On the powered set (Table~\ref{tab:ft-rearm}), the raw fine-tune
exceeds the over-block ceiling in four of five seeds, and fine-tune
plus probe passes in five of five at recall $0.7920 \pm 0.0040$ and
over-block $0.0325$ (the main text rounds this to $0.033$). On the post-fine-tune
$300$-label gate block, base plus probe reproduces the episode's
recorded failure and fine-tune plus probe passes in four of five seeds;
the fifth (trainer seed $456$) falls short by one unsafe item.
The refitted probe makes recall more consistent across the five training seeds.
On the powered evaluation set, the standard deviation is $0.004$ with the probe,
compared with $0.019$ for the fine-tuned classifier alone.

\begin{table*}[t]
\centering\small
\begin{tabular}{@{}l l l l l@{}}
\toprule
System & Recall & 95\% CI & Over-block & Verdict \\
\midrule
\multicolumn{5}{@{}l}{\emph{Powered evaluation set ($2{,}820$ items, $573$ unsafe)}} \\
base classifier & $0.7644$ & $[0.728, 0.797]$ & $0.0467$ & pass \\
base $+$ probe & $0.7836$ & $[0.748, 0.815]$ & $0.0325$ & pass \\
fine-tune, raw (5 seeds) & $0.8028$--$0.8517$ & --- & $0.0654$--$0.0846$ & fail $4/5$ (over-block) \\
fine-tune $+$ probe (5 seeds) & $0.7920 \pm 0.0040$ & --- & $0.0325$ & pass $5/5$ \\
\midrule
\multicolumn{5}{@{}l}{\emph{Post-fine-tune gate block ($300$ items, $47$ unsafe; $2.13$\,pp per unsafe item)}} \\
base classifier & $0.7234$ & $[0.582, 0.831]$ & $0.0316$ & fail \\
base $+$ probe & $0.7447$ & $[0.605, 0.847]$ & $0.0119$ & fail \\
fine-tune, raw (5 seeds) & $0.7447$--$0.8085$ & --- & $0.0474$--$0.0553$ & pass $4/5$ \\
fine-tune $+$ probe (5 seeds) & $0.7660 \pm 0.0150$ & --- & $0.0150$ & pass $4/5$ \\
\bottomrule
\end{tabular}
\caption{\textbf{On the powered set the raw fine-tune exceeds the over-block ceiling in four of five
seeds; after the re-arm on
one fresh $300$-label audit, fine-tune plus probe meets the
acceptance standard (recall floor, over-block ceiling).} Each row is one system; the Verdict column reports pass or fail
against the acceptance standard. Both test sets sit at the alarm-time
mixture and are disjoint from the fitting and validation sets used for
these post-fine-tune comparisons. The acceptance standard is recall
$\geq 0.7605$ and over-block $\leq 0.0692$; it derives from the
pre-drift reading with the $5$\,pp tolerance on each axis. Multi-seed
rows report the range or the mean $\pm$ s.d.\ over the five trainer
seeds.}\label{tab:ft-rearm}
\end{table*}

Base plus probe also passes on the powered set: the episode's recorded
failure is a granularity effect of the $300$-label gate block, so this
experiment demonstrates that the deployed-budget gate outcome reverses
on the powered set; it does not establish that the fine-tune was
necessary for this
episode.
On the $1{,}290$-item retention slice (outside the episode's labels and
never entering the stream), $4$ to $8$
items flip from correct to incorrect and $3$ to $8$ flip the other way
per trainer seed, a net change
between $-4$ and $+1$ items, at most $0.4\%$ of the slice.

\RCVAppendixSection{The Five Chain Records}\label{app:chain-records}

Table~\ref{tab:chain-records} gives the per-cycle record behind the main
text's consecutive-campaign chains. The chains test whether the loop withstands
campaigns on one deployment and what those campaigns cost in labels.
Each cycle's cost is computed as its
audit labels plus the episode's held-out $300$-label gate block, so a
first-audit repair costs $600$. An escalated terminal has exhausted the audit budget,
$300$ fresh labels per attempt over four attempts, plus its gate block:
$1{,}500$ labels. Cycle 1 is judged against the deployment-start
standard. After an accepted repair, the accepted candidate's gate
reading becomes the standard for the next cycle.

\begin{table*}[t]
\centering\small
\setlength{\tabcolsep}{1pt}
\begin{tabular}{@{}lllccc@{}}
\toprule
Chain (seed) & Cycles (family, label cost) & Terminal & Depth & Final attack share & Total labels \\
\midrule
42   & mental health 600; sensitive info 600; cyber 900; stereotypes 600 & fraud, no alarm & 4 & 0.544 & 2{,}700 \\
\rowcolor{gray!10}
789  & disinformation 600; fraud 600; violence 600 & material harm, no alarm & 3 & 0.425 & 1{,}800 \\
1024 & toxic 600; defamation 600; material harm 600 & violence escalated & 3 & 0.383 & 3{,}300 \\
\rowcolor{gray!10}
456  & material harm 900; sensitive info 600 & fraud escalated & 2 & 0.316 & 3{,}000 \\
123  & --- & stereotypes escalated & 0 & 0.000 & 1{,}500 \\
\bottomrule
\end{tabular}
\caption{\textbf{The five chain records.} Each chain sequences campaigns
on one deployment: the Cycles column lists accepted repairs in cycle order
with their label costs; the remaining columns give the terminal event, the
chain depth (accepted repairs), the attack share of the chain's final
reference mixture, and the chain's total label cost. An escalated terminal
costs the full $1{,}500$ labels; a no-alarm terminal triggers no audit and
costs none.}
\label{tab:chain-records}
\end{table*}

\RCVAppendixSection{Off-the-Shelf Drift Monitors Under a Matched Allowance}\label{app:battery-floor}

Under the shared nominal false-alarm allowance of $0.30$, the deployed
monitor and score-KS detect all campaigns. The deployed monitor alarms
at a lower median attack rate than score-KS, but produces more alarms on
drift-free traffic. This pattern is consistent with the usual trade-off
between early detection and false-alarm control.
We compare it against score-KS and ATC on $100$ constructed
attack streams (ten held-out harm families, ten seeds) and $90$
seed-instantiations of the shared drift-free null stream
(Table~\ref{tab:battery-floor}). Every monitor is
calibrated on the same drift-free material. Neither baseline ships as a
sequential monitor for this setting, so both enter the table in adapted
form. Score-KS compares the classifier's unsafe probability $s$ in
sliding windows with the drift-free reference. We use ATC with the
maximum-confidence score $\max(s,1-s)$ proposed by \citet{garg2022}.
To obtain sequential alarms, we treat items below its source-calibrated
threshold as events for the same accumulation procedure.
ATC alarms on $90/90$ drift-free instantiations. The low-confidence
share it monitors ran $0.061$ on its calibration prefix and $0.073$ on
the live null traffic. The replay-calibrated threshold does not cover an
offset of that size.
The $0.05$ allowance is the
conventional level for a stand-alone test, not one of this paper's
design constants. At that level, ATC with its threshold set from the
point estimate of its reference rate exceeds the allowance ($87/90$).
With the threshold set by the same replay design, the realized rate is
$8/90$, consistent with that allowance at this sample size. At the
matched allowance both threshold treatments alarm on every
instantiation. All three monitors use the same drift-free source material
and nominal false-alarm allowance. Our ATC adaptation follows the deployed
monitor's accumulation rule, but uses a fixed detectable shift and a
different procedure for calibrating alarm thresholds. Score-KS instead
compares sliding windows of classifier scores with the drift-free reference.

\begin{table}[t]
\centering\small
\setlength{\tabcolsep}{2pt}
\begin{tabular}{@{}lccc@{}}
\toprule
Monitor & Detections & First-alarm rate & Null alarms \\
 & /100 & median [IQR] & /90 \\
\midrule
score-KS & 100 & 0.134 [0.134, 0.174] & 0 \\
\rowcolor{gray!10}
ATC & --- & --- & 90 \\
deployed monitor & 100 & 0.115 [0.094, 0.137] & 14 \\
\bottomrule
\end{tabular}
\caption{\textbf{The deployed monitor and score-KS stay within the matched
allowance;
ATC does not.} Each row is one monitor: detections count alarmed attack
streams of the $100$; the first-alarm rate is the attack rate at the first
alarm (lower is earlier), median [interquartile range] over the $100$
streams (ten seeds
per held-out family). Thresholds are read from replayed drift-free streams at a nominal
false-alarm allowance of $0.30$ (Appendix~\ref{app:recipes} gives the
deployed monitor's design);
null alarms are counted on the shared $90$ drift-free null streams. ATC's
row reports no detection reading because it alarms on $90/90$ drift-free
streams.}
\label{tab:battery-floor}
\end{table}

\RCVAppendixChapter{Reproducibility}\label{app:repro}

Supplementary code and data: \url{https://rcv.tsandoval.com}.\footnote{\raggedright Repository
snapshot: \url{https://github.com/thiago-cs-ai/regime-conditional-verification/tree/4a4cbb61e49379d5ffeea856b68febb8d16d6dd2}.}

\RCVAppendixSection{Recipes, Seeds, and Splits}\label{app:recipes}

\paragraph{Extraction recipes.} Each classifier emits its verdict on the full
(prompt, response) pair. Only Llama-Guard-3 reads $Z$ from the verdict's own
forward pass; the other two reads are two-pass (below).
\begin{itemize}\itemsep2pt
\item \emph{Llama-Guard-3-8B} ($d{=}4{,}096$): the chat template over the user
prompt and assistant response with the generation prompt appended. $Z$ is the
final-layer decision-token hidden state; the verdict is the softmax over the
safe and unsafe tokens thresholded at $0.5$. Re-extraction under this recipe
reproduces the $Z$ used here at cosine $0.99996$.
\item \emph{WildGuard-7B} ($d{=}4{,}096$): the model's three-field template.
The read is two-pass: pass one freely generates the assessment, and pass two
teacher-forces on the model's own pass-one tokens to re-read the decision-token
hidden state. Greedy decoding makes the pass-two $Z$ faithful to the
free-generation state. We read $Z$ from the final-layer hidden state at the
harmful-response decision token and obtain the verdict from the harmful-response
head. Re-extraction under this recipe
reproduces the $Z$ used here at cosine $0.99997$.
\item \emph{Beaver} ($d{=}5{,}120$): the verdict uses the
\texttt{BEGINNING OF CONVERSATION: USER: \{prompt\} ASSISTANT: \{response\}}
template and thresholds the harm cost at $\ge 3.0$. In a separate pass,
the pair is wrapped in a
\texttt{<prompt>}\,{\ldots}\,\texttt{</prompt>} / \texttt{<response>}\,{\ldots}\,\texttt{</response>}
envelope inside the assistant turn following ``Please provide a response.''
$Z$ is the final hidden state at the last non-pad token (maximum length 512).
Re-extraction under this recipe reproduces the $Z$ used here at cosine $0.99995$.
\end{itemize}

\paragraph{Seeds and splits.} The steering results aggregate ten fixed seeds,
$\{42, 123, 456,$ $789, 1024, 7,$ $2026, 31415,$ $271828, 8675309\}$; the maintenance
census and the monitor comparison aggregate all ten, and the chain and
fine-tune experiments use five. The PKU-SafeRLHF traffic pool is the dataset's
test split, so the cells whose classifier was trained on this corpus are measured
on items held out from that training. Steering reports the mean over seeds and the standard
deviation across them. \ref{app:alien} uses $t$-intervals;
Table~\ref{tab:ft-rearm} uses Wilson intervals.
Steering uses a group-blocked split by prompt,
$45/35/20$ train/calibration/evaluation, with every group (all responses to one
prompt) assigned atomically to a single split. The evaluation split is drawn
first at $20\%$; the configured calibration fraction $0.4375$ applies to the
remainder, yielding the $35\%$ quoted here. Groups whose rows share one
combination of verdict and agreement are allocated within that
combination. Groups whose rows mix combinations form one further pool
and are allocated by the same procedure.
The per-regime logistic-regression probe is fit on the train split; on the
calibration split, each regime's regularisation strength is selected by held-out
AUROC and the calibration map is then fit.\footnote{The calibration map is
a sigmoid from the probe's logit to the probability of agreement, fitted as
L2-regularised logistic regression at $C{=}1$ on the calibration split's hard
labels. Platt's original formulation is unregularised and smooths its targets;
at our calibration-split sizes the two fits are indistinguishable on held-out
log-loss and Brier, and differ by $3\%$ of level in expected calibration error.
Figure~\ref{fig:fig2-method} shows this calibration step for each probe. A calibration split
falling below roughly one hundred rows, or below fifteen minority cases, would
reopen the choice.}

\paragraph{Stream composition.}
\begin{itemize}\itemsep2pt
\item EVAL is the $1{,}709$ WildGuardTest items that include a response;
prompt-only rows are excluded because the unit of analysis is a
prompt--response pair. POOL is about
$10{,}000$ WildGuardTrain items with responses, stratified by adversarial flag
and subcategory, with the drift families excluded.
\item The WildGuardMix steering cells are measured on EVAL and POOL together,
$11{,}708$ items, carved per seed by the split rule above. POOL is
WildGuardTrain, so the WildGuard cells on this corpus are measured partly on
items the classifier was trained on. Any advantage the classifier holds on items it
was trained on reduces its misses there, and with them the probe's room
for correction. The size of that reduction is not measured. The WildGuard rows of
Table~\ref{tab:t1v3-steering} of the main text are measured under that reduced headroom, and the other
two classifiers were not trained on this pool.
The probe fitted here is not reused in
the maintenance experiments, which draw from POOL alone and fit their own
probes on their own slice of it; EVAL enters no maintenance training fold.
\item The pool is carved per seed into probe-training, calibration, retention,
and traffic slices, recorded in deterministic manifests; the stream-assembly
script asserts pairwise disjointness.
\item A drift stream is base traffic from the traffic slice plus one campaign
family, whose rate rises linearly to a cap of $0.30$.
\item The ten held-out campaign families are WildGuardMix's own adversarial
subcategories, excluded from the initial probe-training and calibration
splits. In the WildGuardMix deployment, the first campaign's family is
pre-registered by pool size, the largest subcategory.
\item The retention slice is never streamed and never trained on, and audited
items are never re-audited.
\end{itemize}

\paragraph{Monitor mechanics and parameters.} Within each verdict regime, the monitor counts correctness estimates below
$0.5$ and below $0.95$ separately. The $0.5$ counter sits on the
decision surface, so its events are the flips themselves; the wider $0.95$
counter supplies the sensitivity.
Each counter's reference event rate is estimated from the drift-free run-in and carries a
posterior for its estimation uncertainty. The rate increase the
test is designed to detect, the indifference zone, is the smallest whose break-even
rate sits at quantile $0.99$ of that posterior. Each counter accumulates a
sequential log-likelihood
ratio over its events, one-sided for an increase and measured from its
running minimum, so periods at the reference rate drain it; the monitor
alarms when any counter's accumulation crosses its threshold, and the four
thresholds, two per regime, are searched jointly.
Thresholds are fixed in advance of the monitored horizon by replaying the
run-in: the reference is
estimated on the first $2{,}000$ items of the deployed stream, and $4{,}000$
replay streams of the $19{,}000$-item horizon are drawn from it. Calibration
nests over posterior draws: one reference draw per ten replay streams,
and a threshold set is accepted when the fraction of draws showing any
false alarm is within the allowance. A draw shows a false alarm when
any of its ten replay streams alarms. At the calibration target
$q=0.70$, at least $70\%$ of these $400$ sampled reference draws must be
free of alarms; the complementary $0.30$ is the allowance for this finite
posterior-replay criterion. Three of these values, the
$0.95$ boundary, the target $q$, and the indifference-zone quantile, were
selected on measured curves, on a held-out development family that is excluded
from the ten reported attack families. The $0.95$ boundary is on a
plateau: detection at $0.97$ is identical to three decimals, $0.99$ detects
$28\%$ later, and $0.90$ exceeds the allowance. The target $q$ is the last
value at which unused allowance still converts into detection. The
indifference-zone quantile is at the knee of its measured curve.
The false-alarm evaluation reuses the development family's null stream. On the
$60$ seed instantiations not used during selection, the rate is $10$ of $60$
(Clopper--Pearson $95\%$ interval $[0.083, 0.285]$, whose upper bound is below
the $0.30$ allowance), consistent with the reported $14$ of $90$. The $90$
readings are seed instantiations of one composed stream, not independent
samples of traffic.

\paragraph{Loop constants.} Table~\ref{tab:constants} lists the design
constants, fixed before the maintenance experiments. Each labeled
audit item keeps a fixed assignment to the update's training or calibration
material. Every candidate update, and the pre-repair probe beside it, is
scored on the held-out $300$-label gate block. The audit grows from $300$
to $1{,}200$ labels, and the gate block's $300$ labels are additional.

\begin{table*}[t]
\centering\small
\setlength{\tabcolsep}{6pt}
\begin{tabular}{@{}ll@{}}
\toprule
Design constant & Value \\
\midrule
Thresholds on correctness estimates (per verdict regime) & $0.5$, $0.95$ \\
\rowcolor{gray!10}
Calibration target $q$ / allowance & $0.70$ / $0.30$ \\
Calibration criterion & nested posterior draws ($10$ streams per draw) \\
\rowcolor{gray!10}
Run-in / replays / horizon & $2{,}000$ items / $4{,}000$ streams / $19{,}000$ items \\
Indifference-zone quantile & $0.99$ \\
\rowcolor{gray!10}
Attack-rate cap (linear ramp) & $0.30$ \\
Audit budget (labels per attempt) / attempts per episode & $300$ labels / at most $4$ \\
\rowcolor{gray!10}
Monitor observation chunk & $2{,}000$ items \\
Gate block (held out) & $300$ labels \\
\rowcolor{gray!10}
Acceptance tolerance (recall floor, over-block ceiling) & $5$\,pp each vs.\ the standing standard (the pre-drift reading at cycle 1) \\
\bottomrule
\end{tabular}
\caption{\textbf{Design constants of the monitor and maintenance loop} and
their values, fixed before the maintenance experiments. Each post-alarm audit is
a prefix of a segment composed at the alarm-time
mixture.}\label{tab:constants}
\end{table*}

\begin{table*}[t]
\centering\small
\setlength{\tabcolsep}{6pt}
\begin{tabular}{@{}lll@{}}
\toprule
Stage & Compute & Hardware \\
\midrule
Policy-oracle labeling (PKU-SafeRLHF, one-time) & one run, $16{,}422$ items & GPT-5-Nano API \\
\rowcolor{gray!10}
Policy-oracle labeling (WildGuardMix, one-time) & one run, $11{,}708$ items & GPT-5-Nano API \\
Representation extraction (three classifiers)   & ${\approx}1.5$ h  & L40S / A100 80GB GPU \\
\rowcolor{gray!10}
One fine-tune ($480$ items, one epoch)          & ${\approx}6$ min & A100 80GB GPU \\
Maintenance program (GPU stages)                & a few GPU-hours   & L40S + A100 \\
\rowcolor{gray!10}
Steering, added clause, monitoring                      & ---               & CPU only \\
\bottomrule
\end{tabular}
\caption{\textbf{Per-stage compute.} We extract representations on a single
A100 80GB GPU for PKU-SafeRLHF and a single L40S GPU for WildGuardMix.
The escalation fine-tune runs on a single A100 80GB GPU; the recurring
steering, added-clause, and monitoring compute is CPU-only. Running times are
single-run wall-clock on the stated hardware except the maintenance-program row, which
aggregates the extraction and fine-tune stages across five
seeds.}\label{tab:compute}
\end{table*}

\RCVAppendixSection{The Judge Prompt and Agreement with Human Labels}\label{app:judge}

The deployer's policy oracle $Y^\star$ is a single LLM judge, GPT-5-Nano
(\RCVCode{gpt-5-nano-2025-08-07}), applied to each (prompt, response) pair with
the policy rubric as its system prompt and the response format set to a
JSON object. The OpenRouter request sets the sampling temperature to $0$, but
GPT-5-Nano uses a fixed temperature of $1$. For reproducibility, we store the
resulting labels as a fixed map and reuse them when rerunning the experiments,
without re-querying the judge. We use Cohen's $\kappa$ to measure agreement
with the datasets' own human labels on the binary safe/unsafe decision:
$\kappa=0.838$ on PKU-SafeRLHF (raw agreement
$0.919$, $n{=}16{,}422$) and $\kappa=0.766$ on WildGuardMix (raw agreement
$0.930$, $n{=}11{,}708$).

The policy rubric is a labeling specification, not harmful content; it
ships verbatim in the supplementary code and data package as
\texttt{judge\_prompt.txt}.
It states the two unsafe tests
(baseline harmful content and the deployer's commitment-to-assist
clause) and the strict JSON output schema, whose \texttt{rule\_fired} field
names the rule that fired. Clause membership in Appendix~\ref{app:alien} comes from
this field: an item belongs to the added clause when the
commitment-to-assist rule fired on it.

\RCVAppendixSection{Compute}\label{app:compute}

Most alarms are handled by a probe fit and leave the classifier unchanged; a
fine-tune is reserved for the episodes that do not pass the gate within the
label budget, $21$ of the $100$ in the census. Table~\ref{tab:compute} lists the per-stage compute.
Representation extraction and the fine-tune are the two GPU stages; the
per-regime probe fit, the calibration fit, application of the flip rule, and the
monitor replays run on CPU.
The GPU stages run on single-GPU cloud instances from a pinned container
image, PyTorch~2.4.0 on Python~3.11, CUDA~12.4.1, Ubuntu~22.04, with at
least $8$ vCPUs and $32$\,GB of system memory; task-level packages are
version-pinned by the pod-setup scripts. The CPU stages run
single-machine, and every loop and steering run records its environment,
Python~3.12.3 with NumPy~2.5.1, SciPy~1.18.0, and scikit-learn~1.9.0,
together with SHA-256 digests of the code, the configuration, and the
input data file, in its result's provenance block. The remaining
reproducibility checks are the code-digest match above and the
extraction-recipe cosine checks of Appendix~\ref{app:recipes}.
\end{document}